\documentclass[11pt, a4paper, logo, copyright, nonumbering]{map}
\usepackage{CJKutf8}
\usepackage{xargs}  
\usepackage{todonotes}
\usepackage{amsmath}
\usepackage{dsfont}

\usepackage{longtable}

\usepackage{svg}
\usepackage{fontawesome}
\usepackage{array}
\usepackage{tabularx}
\usepackage{latexsym}
\usepackage{graphicx}
\usepackage[utf8]{inputenc} 
\usepackage[utf8]{inputenc} 
\usepackage{amssymb} 
\usepackage{url}            
\usepackage{amsfonts}       
\usepackage{nicefrac}       
\usepackage{microtype}      
\usepackage{xspace}

\usepackage{listings}
\usepackage{makecell}
\usepackage{inconsolata}
\usepackage{multicol}
\usepackage{amstext}

\definecolor{darkblue}{RGB}{84, 112, 198}

\definecolor{lightblue}{rgb}{0.85, 0.95, 1.0}    
\definecolor{lightgreen}{rgb}{0.90, 1.0, 0.90}    
\definecolor{lightorange}{rgb}{1.0, 0.95, 0.85}   
\definecolor{lightpurple}{rgb}{0.95, 0.90, 1.0}   
\definecolor{lightgray}{rgb}{0.97, 0.97, 0.97}    

\usepackage{tikz}
\usetikzlibrary{calc}
\definecolor{battery-empty}{rgb}{0.9, 0.9, 0.9}
\newcommand{\difficultybar}[1]{%
  \begin{tikzpicture}[baseline, scale=0.5, every node/.style={scale=0.8}]
    \foreach \i in {1,2,3,4,5} {
      \ifnum\i>#1
        \draw[fill=battery-empty] (\i*0.5-0.5, 0) rectangle (\i*0.5, 0.25);
      \else
        \pgfmathsetmacro{\colorlevel}{80 - 12*(\i)} 
        \edef\x{\noexpand\draw[fill=blue!\colorlevel!white, opacity=0.9] (\i*0.5-0.5, 0) rectangle (\i*0.5, 0.25);}
        \x
        \draw[blue!50!black] (\i*0.5-0.5, 0) rectangle (\i*0.5, 0.25);
      \fi
    }
    \fill[battery-empty!70] (2.5, 0.08) rectangle (2.6, 0.17);
    \draw[battery-empty!70!black] (2.5, 0.08) rectangle (2.6, 0.17);
  \end{tikzpicture}%
}
\renewcommand{\arraystretch}{0.96}

\definecolor{hidden-draw}{RGB}{20,68,106}
\definecolor{hidden-pink}{RGB}{255,245,247}
\usepackage[edges]{forest}

\usepackage{bm}

\usepackage{natbib}
\usepackage{CJKutf8}
\usepackage{xargs}
\usepackage{todonotes}
\usepackage{multirow}
\usepackage{cleveref}
\usepackage{subcaption}
\usepackage{url}
\usepackage{svg}
\usepackage{fontawesome}
\usepackage{xcolor}
\usepackage{float}

\usepackage[utf8]{inputenc}
\usepackage{booktabs}
\usepackage{graphicx}
\usepackage{array}
\usepackage{tabularx}
\usepackage{xcolor,colortbl}  
\definecolor{boxcolor}{HTML}{d92523} 
\definecolor{bulbcolor}{HTML}{e3b87f} 
\usepackage{url}
\usepackage{ragged2e}   
\usepackage{makecell} 

\usepackage{hyperref}       
\usepackage{url}            
\usepackage{booktabs}       
\usepackage{nicefrac}       
\usepackage{microtype}      
\usepackage{xcolor}         
\usepackage{graphicx}
\usepackage{longtable}
\usepackage{array}
\usepackage{pbox}
\usepackage{tabularx}
\usepackage{multirow}
\usepackage{wrapfig}
\usepackage{pifont}
\usepackage{lipsum}
\usepackage{subcaption}
\usepackage{enumitem}
\usepackage{tikz}
\usepackage{xstring}

\usepackage{color-edits}

\usepackage{booktabs}
\usepackage{multirow}
\usepackage[table]{xcolor}
\usepackage{tabularx}

\usepackage{parskip} 

\usepackage{threeparttable} 

\usepackage{algorithm}
\usepackage{algpseudocode}

\usepackage{xspace}
\usepackage{times}
\usepackage{latexsym}
\usepackage{booktabs}
\usepackage{fvextra}
\usepackage{graphicx}
\usepackage{subcaption} 
\usepackage[T1]{fontenc}
\usepackage{xcolor}
\usepackage[utf8]{inputenc}

\usepackage{microtype}

\usepackage{inconsolata}
\usepackage{amsmath}
\usepackage{booktabs}
\usepackage{graphicx}
\usepackage{pifont}
\usepackage{multirow}          
\usepackage{colortbl}          
\usepackage{algorithm}
\usepackage{algpseudocode}
\usepackage{xcolor}
\usepackage{amssymb}
\usepackage{amsmath}
\usepackage{amssymb}
\usepackage{amsthm}
\usepackage{tabularx}
\usepackage{array}

\newcolumntype{Y}{>{\raggedright\arraybackslash}X}

\newtheorem{proposition}{Proposition}[section]

\definecolor{rliableolive}{HTML}{BBCC33}
\definecolor{rliableblue}{HTML}{77AADD}
\definecolor{rliablered}{HTML}{f63c44}

\algrenewcommand\algorithmicrequire{\textbf{Input:}}
\algrenewcommand\algorithmicensure{\textbf{Output:}}
\algrenewcommand\alglinenumber[1]{\small #1:}

\usepackage{alltt}
\usepackage[most,skins,theorems]{tcolorbox}
\tcbuselibrary{skins,breakable,fitting,hooks}  
\tcbset{
  aibox/.style={
    width=\linewidth,
    top=7pt,
    bottom=8pt,
    colback=white,
    colframe=black,
    colbacktitle=black,
    enhanced,
    center,
    attach boxed title to top left={yshift=-0.1in,xshift=0.15in},
    boxed title style={boxrule=0pt,colframe=white},
  }
}
\tcbset{
  aibox2/.style={
    width=\linewidth,
    top=7pt,
    bottom=2pt,
    colback=green!18!white,
    colframe=black,
    colbacktitle=black,
    enhanced,
    center,
    attach boxed title to top left={yshift=-0.1in,xshift=0.15in},
    boxed title style={boxrule=0pt,colframe=white,},
  }
}
\newtcolorbox{AIbox}[2][]{aibox,title=#2,#1}
\newtcolorbox{AIbox2}[2][]{aibox2,title=#2,#1}

\usepackage{xcolor}

\hypersetup{
    colorlinks=true,            
    linkcolor=blue,             
    filecolor=magenta,          
    urlcolor=cyan,              
    citecolor=purple,             
    pdftitle={Overleaf Example},
    pdfpagemode=FullScreen,
}

\definecolor{iquestblue}{HTML}{173C7F}
\definecolor{iquestazure}{HTML}{528FCC}

\newcommandx{\info}[2][1=]{\todo[linecolor=red,backgroundcolor=red!25,bordercolor=red,#1]{#2}}

\title{
\vspace{-0.2in}
\centering \fontsize{15pt}{16pt}\selectfont

Harness Scaling for Generalizable Agentic Reinforcement Learning

\vspace{-0.2in}
}

\author{
Fei Bai\textsuperscript{1*},
Huatong Song\textsuperscript{1*},
Chuan Hao\textsuperscript{2},\\
\normalfont
Ran Tao\textsuperscript{2},
Jian Yang\textsuperscript{3},
Bryan Dai\textsuperscript{2},
Wayne Xin Zhao\textsuperscript{1$\dagger$},
Ji-Rong Wen\textsuperscript{1}\\
\normalfont
\textsuperscript{1}Gaoling School of Artificial Intelligence, Renmin University of China, \\
\normalfont
\textsuperscript{2}IQuest Research, \textsuperscript{3}Beihang University \\
\normalfont
\normalsize{\textsuperscript{*}Equal contributors\quad
\textsuperscript{$\dagger$}Corresponding Authors \\
\normalfont
Email: \texttt{\{feibai,songhuatong123\}@ruc.edu.cn}, \\
\normalfont
\texttt{batmanfly@gmail.com}
}}

\newcommand{\ignore}[1]{}

\usepackage[most]{tcolorbox}
\definecolor{darkorange}{RGB}{255, 140, 0}
\definecolor{lightgreen}{RGB}{145, 204, 117}
\definecolor{lightyellow}{RGB}{250, 200, 88}
\definecolor{lightred}{RGB}{238, 102, 102}
\definecolor{lightblue}{RGB}{115, 192, 222}

\definecolor{gray_1}{HTML}{B7B7B7}
\definecolor{gray_2}{HTML}{F0F0F0} 
\definecolor{frame_blue}{HTML}{A9D18E}

\newtcolorbox[auto counter, number within=section]{PromptBox}[2][]{
    enhanced,
    breakable,
    colback=gray_2, 
    colframe=gray_1,
    coltitle=white,
    fontupper=\small,
    fonttitle=\bfseries,
    title={#2}, 
    label={#1},
    arc=2pt,
    boxrule=1pt,
    left=2mm, right=2mm, top=2mm, bottom=2mm,
}

\usepackage{etoolbox}
\makeatletter
\patchcmd{\@maketitle}
  {\@toptitlebar}
  {%
    \hbox to \textwidth{%
    
    \includegraphics[height=0.7cm]{figures/aibox-logo.png}%
      \hfill
      \includegraphics[height=0.7cm]{figures/logo2.pdf}%
    }%
    \vspace{0.3em}
    \@toptitlebar
  }
  {}{}
\makeatother

\begin{abstract}
Agent harnesses have become a central component of modern agentic systems, shaping how models use tools, manage context, and coordinate long-horizon interaction. Realizing the benefits of these harnesses depends on how effectively the underlying model can make use of them, motivating reinforcement learning directly through the harness. While recent efforts have explored harness diversification to reduce specialization to particular execution settings, how to systematically construct and organize harness diversity for cross-harness generalization remains insufficiently understood. This work studies Harness~Scaling as a distinct scaling dimension for agentic RL, focusing on structured variation in how policies execute and interact with the same underlying task and environment. Concretely, we construct a factorized harness space, select a compact training subset with balanced factorial coverage, and organize cross-harness training through locality-aware task assignment. Across multiple benchmarks, model scales, and RL algorithms, Harness Scaling consistently improves generalization to unseen harness combinations. Moreover, policies trained through structured white-box harness variation transfer effectively to unseen external black-box harnesses, outperforming policies trained through a single complex black-box harness. These results establish harness diversity as a meaningful scaling dimension for agentic RL and highlight the importance of how this diversity is organized for cross-harness generalization.  
\end{abstract}

\begin{document}
\maketitle

\let\oldthefootnote\thefootnote




\section{Introduction}

Agent harnesses are playing an increasingly central role in modern
agentic systems~\cite{tang2026agent}. By supporting sustained interaction, tool use, state
maintenance, and recovery from failures~\cite{chen2026failedtrajectoriesreliablellm}, harnesses have substantially
improved agent performance on complex, long-horizon tasks~\cite{yang2024sweagentagentcomputerinterfacesenable,ursekar2026veroharnessagentsoptimize,ma2026longhorizonharnessadvancinglonghorizonagents}. Yet the
benefits of a capable harness depend not only on its design, but also on
whether the underlying policy can use it effectively~\cite{lin2026harnessupdatingharnessbenefit}. This naturally motivates optimizing the policy directly through the harness
with reinforcement learning~\cite{xu2026polar,yu2026openforgerl}, which strengthens its ability to exploit the
available interaction mechanisms during task solving.


However, optimizing a policy through a single fixed harness may
encourage specialization to a particular execution configuration
and weaken transfer across harnesses~\cite{he2026agentlightningv10harnessed,wang2026openclawrl}.
Recent efforts have explored harness diversification through
randomization and training across heterogeneous harnesses~\cite{huang2026katcoderv25technicalreport,song2026clawgymiiexploringblackbox}.
Yet how to systematically construct and organize harness diversity
for generalization beyond the training configurations remains
insufficiently understood.
We therefore study \textbf{Harness Scaling} as a distinct scaling
dimension for agentic RL.
Unlike task or environment
scaling~\cite{song2026envscalerscalingtoolinteractiveenvironments,dong2026agentworldscalingrealworldenvironment,bai2026clawgym,wang2026agentworldmodelinfinity},
which expands what agents solve or where they operate,
Harness Scaling varies how policies execute and interact with
the same underlying task and environment, promoting generalization
across execution configurations.

A key challenge in Harness Scaling is how to model harness variation and select representative configurations under a fixed training budget.
Given the modular
nature of agent harnesses~\cite{he2026agentlightningv10harnessed,shao2026harnessr1learningeditexecutable}, we factorize harness variation along five
components: system prompt, tool interface, observation, context management, and workflow.
Each component contains multiple realizations, whose combinations define a
factorial space of complete harness configurations. As variation along these
components increases, the space expands rapidly, while exhaustive training
can introduce substantial redundancy under a fixed budget. We therefore
preserve low-order factorial structure in a compact training set, balancing
individual component realizations and their pairwise combinations through
an orthogonal design~\cite{rao1947factorial,hedayat1999orthogonal}. During reinforcement learning, the same task is paired
with structurally nearby harness configurations, providing controlled local
variation within each task while maintaining broad coverage across the
training space. 

To demonstrate the effectiveness of Harness Scaling, we evaluate
cross-harness generalization across multiple benchmarks, model scales, and
RL algorithms. By training across diverse harness configurations, Harness
Scaling reduces specialization to any single harness and encourages more
generalizable agent behavior. Empirically, it consistently outperforms
single-harness RL on unseen white-box configurations. 
More importantly, policies trained through structured white-box harness
variation transfer effectively to complex external black-box harnesses~\cite{openclaw2026,anthropic2026claudecode},
outperforming single-harness black-box training.
These results establish harness diversity as a
meaningful scaling dimension for agentic RL, while showing that effective
scaling depends critically on how this diversity is organized during
training. In summary, our contributions are threefold:


\begin{itemize}
    \item We systematically study Harness Scaling as a distinct scaling dimension of agentic RL, complementary to task and environment scaling, and establish a controlled setting to isolate harness variation and evaluate policy generalization beyond the training configurations.



\item We develop a structured framework for Harness Scaling that
factorizes harness variation, selects a compact training set with balanced
factorial coverage, and organizes cross-harness training by assigning
structurally nearby configurations to the same task across the training space.

    \item We demonstrate across benchmarks, model scales, and RL algorithms
that Harness Scaling substantially improves generalization to unseen
white-box and external black-box harnesses, establishing both the value
of scaling across harnesses and the importance of organizing this scaling
structurally.
\end{itemize}
\section{Related Work}

\subsection{Agentic Reinforcement Learning}

Reinforcement learning has emerged as an important paradigm for improving
the capabilities of LLMs~\cite{zhao2023survey,deepseekmath}. Recent work has extended RL from single-turn
reasoning to increasingly complex agentic settings, including tool use~\cite{bai2026towards,song2025r1}, web
navigation~\cite{qi2025webrltrainingllmweb}, software engineering~\cite{song2026swe}, and professional workflows~\cite{wang2026cuagymscalingverifiabletraining}. As agentic RL moves toward broader
capability, recent efforts have increasingly scaled training along the
\emph{task} and \emph{environment} dimensions. Task scaling~\cite{sun2025simpledeepsearcher,bai2026clawgym,deng2026fort} broadens the
distribution of problems used for policy optimization, while environment
scaling~\cite{dong2026agentworldscalingrealworldenvironment,song2026envscalerscalingtoolinteractiveenvironments,wang2026agentworldmodelinfinity} expands the range of external settings in which agents interact and
receive feedback. These directions expand what agents learn to
solve and where they learn to operate, while often retaining a fixed
execution harness. In contrast, we study harness diversity as an additional
scaling dimension for agentic RL, focusing on policy learning and
generalization across execution configurations.

\subsection{Harness Design and Optimization}


The design of an agent harness shapes how language models use tools, manage context, and perform long-horizon tasks~\cite{yang2024sweagentagentcomputerinterfacesenable,lin2026harnessupdatingharnessbenefit,ma2026longhorizonharnessadvancinglonghorizonagents}. Early systems relied on manually designed harnesses, while later work explored automated improvement through search, editing, and evolution~\cite{lin2026agenticharnessengineeringobservabilitydriven,lee2026recursiveharnessselfimprovement,shao2026harnessr1learningeditexecutable}. Another line optimizes model policies through agent harnesses, including black-box reinforcement learning~\cite{xu2026polar,he2026agentlightningv10harnessed,luo2026harnessbanksemanticgenebanksearch,ning2026evoharnessrllearningselfevolvingruntime}. Recent studies have explored harness diversification through execution variation, composable configurations, and heterogeneous training~\cite{huang2026katcoderv25technicalreport,chen2026occamy10openparetofrontier35b,song2026clawgymiiexploringblackbox}. Building on this direction, our work focuses on how harness diversity can be selected and organized for cross-harness generalization. Specifically, we construct a controlled combinatorial harness space, combine balanced factorial selection with locality-aware task assignment, and examine their effects through controlled ablations. We further evaluate generalization to unseen harness combinations and external black-box harnesses.

\section{Method}

\begin{figure}[t]
    \centering
    \includegraphics[width=1.0\linewidth]{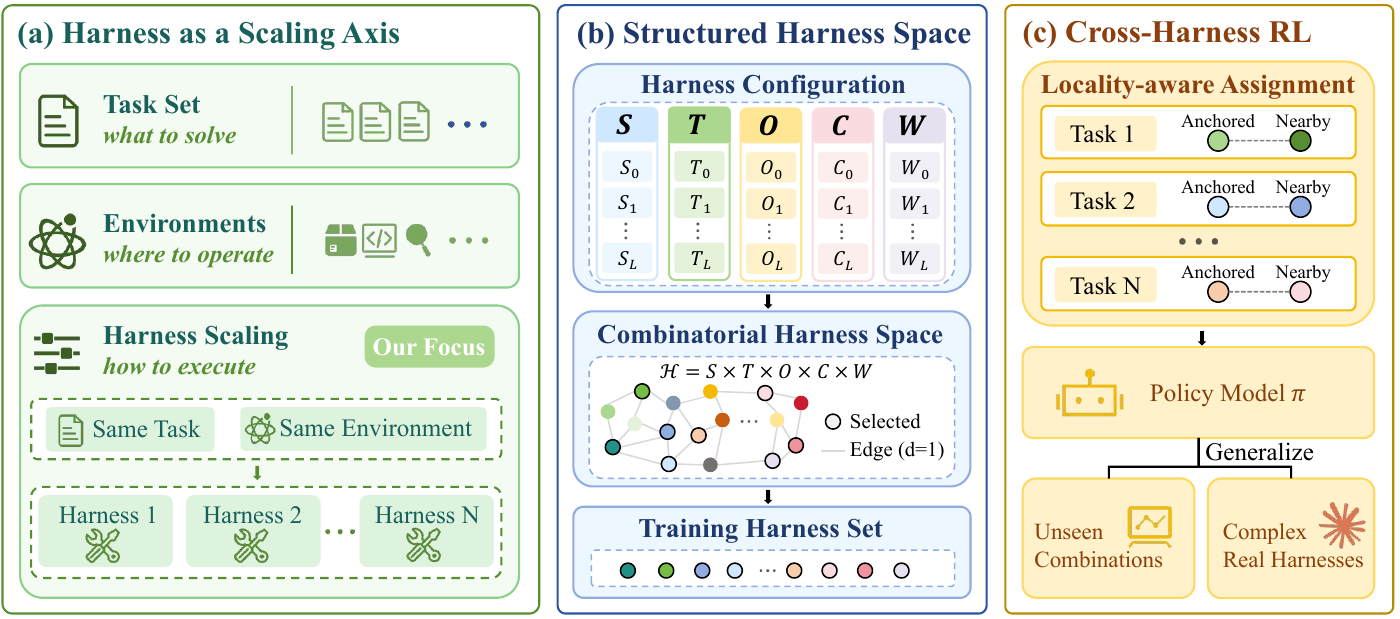}
    \caption{
Overview of Harness Scaling, from structured harness-space construction
to cross-harness training and generalization to unseen configurations.
}
    \label{fig:main}
\end{figure}

\subsection{Problem Formulation}
To study harness scaling systematically, we first formalize the space of
execution configurations and the generalization problem defined over it.
We construct a factorized harness space, define a structural notion of
proximity among its configurations, and formulate the generalization
objectives that motivate the structured harness selection and multi-harness
training procedures shown as in Figure~\ref{fig:main}.

\subsubsection{Factorized Harness Space}

A harness serves as the execution layer through which a policy interacts
with its environment. Rather than treating it as a monolithic implementation,
we factorize the harness into five components: \textit{system prompt},
\textit{tool interface}, \textit{observation}, \textit{context management},
and \textit{workflow}, denoted by \(S,T,O,C,\) and \(W\), respectively.
These components govern the instructions provided to the policy, tool access,
environment feedback, interaction history, and execution procedure,
collectively shaping how the policy perceives, acts, and maintains state
throughout task execution. This factorization provides an operational
decomposition of the harness, while allowing the effects of different
components to interact.

We instantiate each component through a family of component-specific
transformations. For a component \(F\in\{S,T,O,C,W\}\), let \(F^{(0)}\)
denote its base realization and \(\tau_F^{(\ell)}\) the \(\ell\)-th
transformation applied to it:
\begin{equation}
F^{(\ell)}
=
\tau_F^{(\ell)}\left(F^{(0)}\right).
\label{eq:component-transformation}
\end{equation}
A transformation modifies only the target component while preserving the
remaining components, task semantics, and environment. For example, it may
reorganize the system prompt, alter tool schemas or naming conventions,
change observation feedback, modify context retention or compression, or
impose a different execution procedure. Detailed definitions and
implementations are provided in Appendix~\ref{app:harness_variants}.

Let \(\mathcal{S}\), \(\mathcal{T}\), \(\mathcal{O}\), \(\mathcal{C}\),
and \(\mathcal{W}\) denote the resulting realization sets. A complete harness
configuration selects one realization from each set, defining the factorized
harness space
\begin{equation}
\mathcal{H}
=
\mathcal{S}
\times
\mathcal{T}
\times
\mathcal{O}
\times
\mathcal{C}
\times
\mathcal{W}.
\label{eq:harness-space}
\end{equation}
Each configuration \(h\in\mathcal{H}\) therefore specifies a complete execution
setting for the same underlying task and environment.

The product structure of \(\mathcal{H}\) also induces a natural notion of
structural proximity. For \(h_i,h_j\in\mathcal{H}\), the Hamming distance
\(d_{\mathrm H}(h_i,h_j)\) counts the number of component dimensions on
which they differ. Equivalently, \(\mathcal{H}\) forms a Hamming graph whose
adjacent configurations differ in exactly one component. We use this
distance to organize local cross-harness variation during training.

\subsubsection{Generalization Objective}

The central objective of Harness Scaling is to learn policies that remain
effective beyond the harness configurations observed during training. We
consider two complementary forms of cross-harness generalization: broad
combinatorial generalization within the constructed harness space and transfer
to complex real-world black-box harnesses beyond it.

\paragraph{In-Space Combinatorial Generalization.}
Let \(\mathcal{H}_{\mathrm{tr}}\subset\mathcal{H}\) denote the harness
configurations used for training. We evaluate on held-out configurations in
\(\mathcal{H}\setminus\mathcal{H}_{\mathrm{tr}}\), whose complete
combinations are unseen during training. This setting provides broad coverage
of unseen combinations within the controlled factorized space.

\paragraph{Out-of-Space Generalization.}
We further evaluate transfer to complex real-world black-box harnesses whose
execution configurations are not represented in \(\mathcal{H}\). This setting
examines whether training over structured white-box harness variation can
transfer to external execution systems beyond the constructed design space.

\subsection{Harness Scaling with Structured Selection}

As the factorized harness space grows combinatorially with its component
dimensions, exhaustive training can introduce substantial redundancy.
We therefore construct a compact training set that preserves broad factorial
coverage under a fixed optimization budget.

\subsubsection{Factorial Coverage}

The harness space in Eq.~\ref{eq:harness-space} can be viewed as a
factorial design, where each harness component corresponds to a factor
and its realizations correspond to different levels~\cite{rao1947factorial,hedayat1999orthogonal}. To construct a compact
subset without losing basic coverage of this combinatorial space, we preserve
its low-order factorial structure: the representation of individual component
realizations and their pairwise combinations. These correspond to the
first- and second-order projections of the full factorial space, respectively.

\paragraph{Component Balance.}
We first require the realizations of each component to be represented
evenly across the selected configurations. This ensures coverage along
every component dimension without overrepresenting particular realizations.

\paragraph{Pairwise Combination Balance.}
Component balance alone does not control how realizations from different
components are combined. We therefore additionally balance all pairwise
combinations. Indexing the five component dimensions by
\(i,j\in\{1,\ldots,5\}\), let \(N_{ij}(\ell,\ell')\) denote the number
of selected configurations containing the \(\ell\)-th realization of
component \(i\) together with the \(\ell'\)-th realization of component
\(j\). Pairwise balance requires
\begin{equation}
N_{ij}(\ell,\ell')=\lambda,
\label{eq:pairwise-balance}
\end{equation}
where \(\lambda\) is the common occurrence count for every realization
pair of any \(i\neq j\). In factorial-design terminology, this corresponds
to strength-2 balance, under which every pair of realizations from any two
components appears equally often in the selected set.

\subsubsection{Harness Set Construction}
\label{sec:Harness Set Construction}

We construct the training harness set using a strength-2 orthogonal design
over the component realizations in \(\mathcal{H}\). This provides a direct
construction of a compact subset satisfying the component-level and pairwise
balance defined above, without searching over possible subsets of the full
harness space.

The size of such a design is constrained by its pairwise coverage.
Suppose there are \(M\) component dimensions, each with \(L\) realizations,
and every pairwise realization combination occurs \(\lambda\) times.
Since any two components admit \(L^2\) realization combinations, the number
of selected configurations \(N\) satisfies
\begin{equation}
N = \lambda L^2.
\label{eq:oa-size}
\end{equation}
Moreover, a strength-2 orthogonal design must satisfy the Rao
bound~\cite{rao1947factorial,hedayat1999orthogonal},
\begin{equation}
N \geq 1 + M(L-1).
\label{eq:rao-bound}
\end{equation}
Pairwise balance determines the replication structure of the design, while
the Rao bound provides a lower bound on the number of configurations.
These conditions restrict the admissible design sizes.

We instantiate a strength-2 orthogonal design satisfying these constraints,
yielding the training harness set
\(\mathcal{H}_{\mathrm{tr}}\subset\mathcal{H}\). The resulting set preserves
exact component-level and pairwise balance using only a compact subset of the
full factorial space. While these low-order projections are balanced, many
complete harness combinations remain unseen during training, enabling
evaluation of in-space combinatorial generalization. Further details of the
construction and the resulting training configurations are provided in
Appendix~\ref{app:selection_details}.

\subsection{Reinforcement Learning across Harnesses}

Training each task under multiple harnesses exposes the policy to different
execution configurations for the same underlying problem, encouraging the
learned behavior to generalize beyond a single harness. We realize this over
the selected harness set \(\mathcal{H}_{\mathrm{tr}}\) by pairing each task
with a small group of structurally related configurations. The assignment is
designed to combine controlled variation within each task with broad coverage
across the training harness space, and the resulting task--harness instances
form the basis for rollout collection and policy optimization.

\subsubsection{Cross-Harness Task Assignment}

We use a locality-aware strategy to organize how harness configurations are
assigned to tasks. For each task \(x\), we assign \(K\) configurations from
\(\mathcal{H}_{\mathrm{tr}}\). 
An anchor harness
\(h_x \sim \mathrm{Unif}(\mathcal{H}_{\mathrm{tr}})\) is first sampled
uniformly from the training set, and the remaining \(K-1\) configurations
are chosen from \(\mathcal{H}_{\mathrm{tr}} \setminus \{h_x\}\) as those
with the smallest Hamming distance \(d_{\mathrm H}(h_x,h)\) to the anchor.
This associates each task with a
small set of nearby configurations while allowing different tasks to cover
different regions of the training harness space. The resulting harness set
for task \(x\) is
\begin{equation}
\mathcal{H}_x
=
\left\{
h_x,\,
h_x^{(1)},\ldots,h_x^{(K-1)}
\right\},
\label{eq:task-harness-assignment}
\end{equation}
where \(h_x^{(1)},\ldots,h_x^{(K-1)}\) denote the \(K-1\) nearest
configurations to \(h_x\) under \(d_{\mathrm H}\).

This assignment combines local cross-harness variation with broad coverage
over training. For a given task, nearby harnesses provide related but distinct
execution configurations, while random anchor selection distributes different
tasks throughout \(\mathcal{H}_{\mathrm{tr}}\). Each task therefore yields
\(K\) task--harness instances for rollout collection and policy optimization.
When \(K=1\), only the sampled anchor is assigned to the task, while different
tasks remain distributed across \(\mathcal{H}_{\mathrm{tr}}\).

\subsubsection{Multi-Harness Policy Optimization}

Each task--harness instance \((x,h)\) is executed under its assigned
harness, producing a trajectory \(\tau\) and its reward.
The \(K\) harnesses assigned to the same task are used over \(K\)
successive training steps, allowing the policy to receive learning signals
from different execution configurations for the same task.

For a policy-update batch \(\mathcal{B}\) induced by the task--harness
training distribution \(\mathcal{D}_{\mathrm{MH}}\), let
\(N_{\mathcal{B}}\) denote the total number of valid policy-generated
tokens in the batch. We optimize the standard clipped policy objective:
\begin{equation}
\mathcal{L}(\theta)
=
\frac{1}{N_{\mathcal{B}}}
\sum_{(x,h,\tau)\in\mathcal{B}}
\sum_{t=1}^{|\tau|}
M_t
\min\left(
\rho_t(\theta)A_t,\,
\operatorname{clip}\bigl(\rho_t(\theta),1-\epsilon,1+\epsilon\bigr)A_t
\right),
\label{eq:multi-harness-policy-optimization}
\end{equation}
where
\(N_{\mathcal{B}}
=\sum_{(x,h,\tau)\in\mathcal{B}}\sum_{t=1}^{|\tau|}M_t\),
\(\rho_t(\theta)=
\frac{\pi_{\theta}(\tau_t\mid x,h,\tau_{<t})}
{\pi_{\theta_{\mathrm{old}}}(\tau_t\mid x,h,\tau_{<t})}\),
and \(M_t\in\{0,1\}\) masks valid policy-generated tokens.
The advantage \(A_t\) is computed according to the underlying
reinforcement learning algorithm.

Harness Scaling determines the task--harness instances used for rollout
collection, while the optimization within each step follows the original
RL objective. 
We apply the same cross-harness training procedure to both
PPO~\cite{schulman2017proximal} and GRPO~\cite{deepseekmath}.

\section{Experiments}
In this section, we first detail the experimental setup
and then report the results and detailed analysis.

\subsection{Experiment Settings}


\paragraph{Harness Scaling Setup.}
We instantiate five harness components with three realizations each,
yielding \(3^5=243\) configurations in the full harness space and forming
the Hamming graph \(H(5,3)\). Following the strength-2 orthogonal-design
construction in Section~\ref{sec:Harness Set Construction}, the resulting
design contains 18 configurations, which form the training harness set
\(\mathcal{H}_{\mathrm{tr}}\).

\paragraph{Training.}
We conduct experiments using two backbone models, Qwen3-4B-Instruct and
Qwen3-30B-A3B-Instruct~\cite{yang2025qwen3technicalreport}, on ClawGym-Task~\cite{bai2026clawgym}, which contains 13.5K
workspace-based general agent tasks. We consider both GRPO~\cite{deepseekmath} and PPO~\cite{schulman2017proximal}.
For GRPO, each step uses 16 tasks with 8 rollouts per task, while PPO uses
128 tasks with one rollout per task. We set \(K=2\) harnesses per task and train for 300 optimization
steps. Additional implementation details
and hyperparameters are provided in Appendix~\ref{app:training}.

\begin{table*}[t]
\centering
\caption{
Pass@1 performance across three benchmarks under different harness
evaluation settings. \textbf{In-Space} reports the average over 18
held-out harness configurations, while \textbf{Claude Code} denotes
out-of-space evaluation with the external black-box harness.
\textbf{Bold} indicates the better result between single-harness training
and harness scaling under the same model and optimization method.
}
\label{tab:main_results}

\footnotesize
\setlength{\tabcolsep}{4.2pt}
\renewcommand{\arraystretch}{1.12}

\begin{tabular*}{0.9\textwidth}{
@{\extracolsep{\fill}}llccccccc@{}
}
\toprule
\multirow{2}{*}{\textbf{Method}}
& \multirow{2}{*}{\textbf{Harness}}
& \multicolumn{2}{c}{\textbf{ClawGym-Bench}}
& \multicolumn{2}{c}{\textbf{JobBench-Easy}}
& \multicolumn{2}{c}{\textbf{PinchBench}}
& \multirow{2}{*}{\textbf{Avg.}} \\
\cmidrule(lr){3-4}
\cmidrule(lr){5-6}
\cmidrule(lr){7-8}
&
&
\textbf{In-Space}
& \textbf{Claude Code}
& \textbf{In-Space}
& \textbf{Claude Code}
& \textbf{In-Space}
& \textbf{Claude Code}
& \\
\midrule

\rowcolor{gray!12}
\multicolumn{9}{c}{\textbf{Qwen3-4B-Instruct}} \\
\addlinespace[1pt]

Base
& -- & 18.08 & 33.90 & 7.16 & 6.78 & 32.69 & 32.74 & 21.89 \\

\cmidrule{1-9}

\multirow{2}{*}{GRPO}
& Single
& 37.33 & 38.96 & 10.24 & 5.79 & 45.60 & 40.95 & 29.81 \\
& Scaling
& \textbf{52.78} & \textbf{45.37} & \textbf{13.55}
& \textbf{14.83} & \textbf{56.24} & \textbf{51.47}
& \textbf{39.04} \\

\cmidrule(lr){2-9}

\multirow{2}{*}{PPO}
& Single
& 43.78 & 40.75 & 12.63 & 9.68 & 48.86 & 47.74 & 33.91 \\
& Scaling
& \textbf{49.07} & \textbf{44.36} & \textbf{13.16}
& \textbf{13.39} & \textbf{53.75} & \textbf{47.82}
& \textbf{36.93} \\

\midrule

\rowcolor{gray!12}
\multicolumn{9}{c}{\textbf{Qwen3-30B-A3B-Instruct}} \\
\addlinespace[1pt]

Base
& -- & 34.80 & 37.56 & 8.44 & 7.04 & 49.11 & 51.48 & 31.40 \\

\cmidrule{1-9}

\multirow{2}{*}{GRPO}
& Single
& 48.12 & 46.02 & 14.79 & 10.05 & 54.51 & 50.25 & 37.29 \\
& Scaling
& \textbf{58.60} & \textbf{49.62} & \textbf{15.48}
& \textbf{14.20} & \textbf{58.98} & \textbf{59.78}
& \textbf{42.78} \\

\cmidrule(lr){2-9}

\multirow{2}{*}{PPO}
& Single
& 44.23 & 45.17 & 13.98 & 12.11 & 50.54 & 51.85 & 36.31 \\
& Scaling
& \textbf{48.29} & \textbf{48.83} & \textbf{17.56}
& \textbf{18.90} & \textbf{57.59} & \textbf{60.18}
& \textbf{41.89} \\

\bottomrule
\end{tabular*}
\end{table*}

\paragraph{Evaluation.}
We evaluate on three benchmarks with varying domains and difficulty levels.
ClawGym-Bench~\cite{bai2026clawgym} covers workspace-based general agent tasks, PinchBench~\cite{pinchbench2026} focuses
on realistic tasks involving research, coding, scheduling, and file
operations, and JobBench-Easy~\cite{li2026jobbenchaligningagentwork} evaluates more challenging professional
workflows over heterogeneous workspace files.
All evaluation tasks are disjoint from the ClawGym-Task training set.
For in-space generalization, we evaluate on 18 fixed held-out configurations
from \(\mathcal{H}\setminus\mathcal{H}_{\mathrm{tr}}\), selected for balanced
component coverage and diversity across complete harness configurations.
The same held-out set is used across all benchmarks, model scales, and RL
algorithms. For out-of-space generalization, we use Claude Code~\cite{anthropic2026claudecode} as the main
external black-box harness and additionally evaluate on OpenClaw~\cite{openclaw2026}.
We report Pass@1 averaged over three independent evaluation runs, with
in-space scores further averaged across the 18 held-out configurations.
Further evaluation details and OpenClaw results are provided in
Appendix~\ref{app:evaluation}.


\textbf{Baselines.}
We compare against four baselines: the original backbone model without
additional training; single-harness RL, where the policy is optimized
through the fixed base configuration
\(h_{\mathrm{base}}=(S^{(0)},T^{(0)},O^{(0)},C^{(0)},W^{(0)})\);
black-box SFT, which fine-tunes the policy on trajectories collected through
a complex black-box harness; and black-box RL, which optimizes the policy
through the same black-box harness. We use OpenClaw as the black-box training
harness, while Claude Code is reserved as an unseen external harness for
evaluation. All RL baselines use matched optimization budgets where
applicable.

\begin{table*}[t]
\small
\centering
\caption{
Pass@1 performance on Claude Code as an unseen external black-box harness.
Black-Box SFT and Black-Box RL are trained through OpenClaw, while Harness
Scaling is trained over the structured white-box harness space. For the
RL-based methods, both Black-Box RL and Harness Scaling use GRPO.
}
\label{tab:blackbox_comparison}

\footnotesize
\setlength{\tabcolsep}{8pt}
\renewcommand{\arraystretch}{1.12}

\begin{tabular}{@{}lcccc@{}}
\toprule
\textbf{Method}
& \textbf{ClawGym-Bench}
& \textbf{JobBench-Easy}
& \textbf{PinchBench}
& \textbf{Avg.} \\
\midrule

\rowcolor{gray!12}
\multicolumn{5}{c}{\textbf{Qwen3-4B-Instruct}} \\

Black-Box SFT
& 44.48 & 11.10 & 49.66 & 35.08 \\

Black-Box RL
& 44.10 & 9.42 & 46.61 & 33.38 \\

Harness Scaling
& \textbf{45.37}
& \textbf{14.83}
& \textbf{51.47}
& \textbf{37.22} \\

\midrule

\rowcolor{gray!12}
\multicolumn{5}{c}{\textbf{Qwen3-30B-A3B-Instruct}} \\

Black-Box SFT
& 49.55 & 13.16 & 51.52 & 38.08 \\

Black-Box RL
& 47.84 & 12.73 & 48.28 & 36.28 \\

Harness Scaling
& \textbf{49.62}
& \textbf{14.20}
& \textbf{59.78}
& \textbf{41.20} \\

\bottomrule
\end{tabular}
\end{table*}


\subsection{Main Results}

\noindent $\bullet$ \emph{Combinatorial Generalization to Unseen Harnesses.}
Compared with Single-Harness RL, Harness Scaling consistently improves
performance on held-out harness configurations whose complete combinations
are unseen during training, as summarized in Table~\ref{tab:main_results}
and further illustrated by Figure~\ref{fig:eval_curve_whitebox}.
Averaged over the three benchmarks, the in-space gains reach \(+9.80\)
and \(+3.57\) points with GRPO and PPO on Qwen3-4B-Instruct, and
\(+5.21\) and \(+4.90\) points on Qwen3-30B-A3B-Instruct.
The advantage also persists when Single-Harness RL is trained on three
additional fixed harnesses selected solely by their structural positions
in \(\mathcal H_{\mathrm{tr}}\). Detailed sensitivity results are reported
in Appendix~\ref{app:addition_abla}.
These results support the benefit of Harness Scaling for combinatorial
generalization within the constructed harness space.

\begin{wrapfigure}{r}{0.52\columnwidth}
    \vspace{-6pt}
    \centering

    \begin{minipage}[t]{0.48\linewidth}
        \centering
        \includegraphics[width=\linewidth]
        {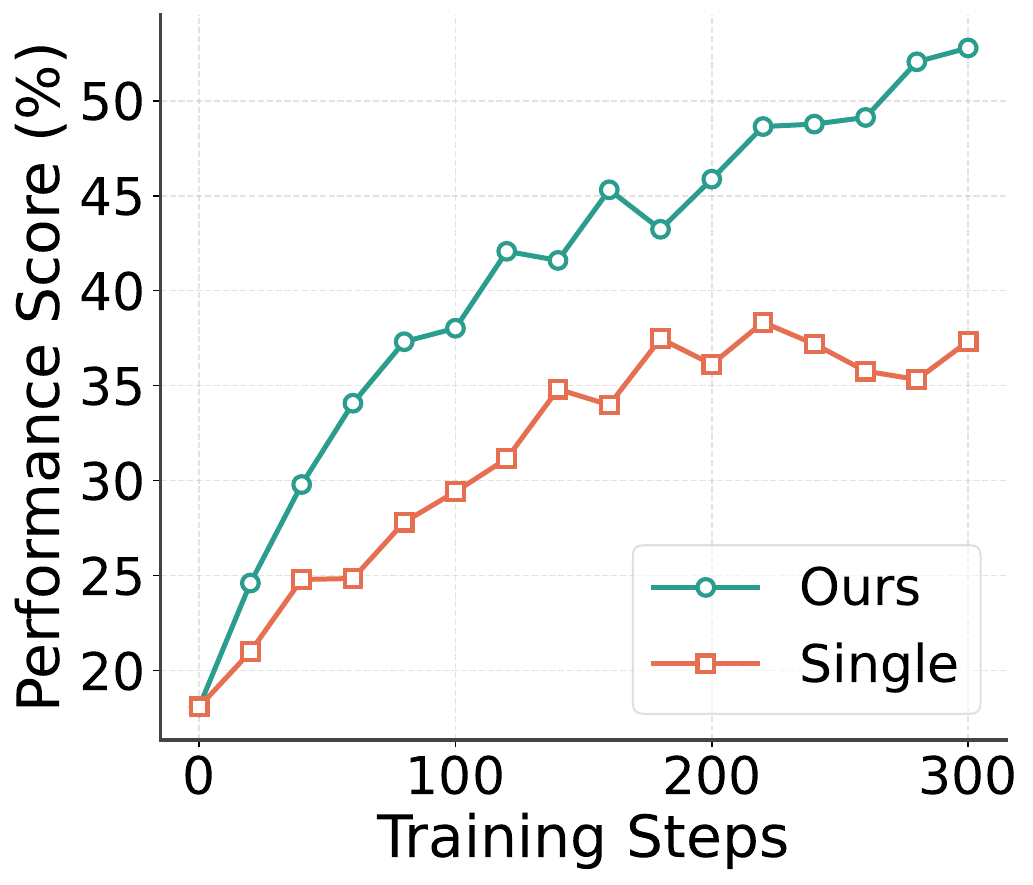}

        \vspace{-2pt}
        {\scriptsize (a) Qwen3-4B-Instruct}
    \end{minipage}
    \hfill
    \begin{minipage}[t]{0.48\linewidth}
        \centering
        \includegraphics[width=\linewidth]
        {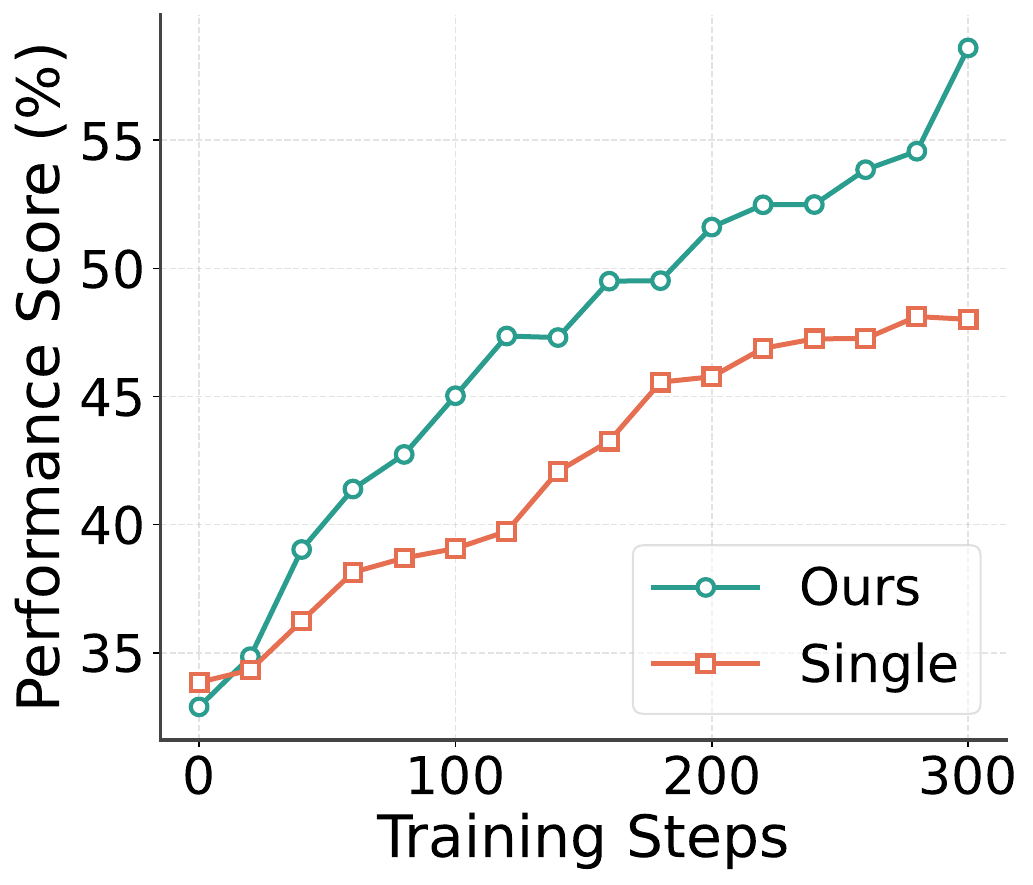}

        \vspace{-2pt}
        {\scriptsize (b) Qwen3-30B-A3B}
    \end{minipage}

    \vspace{-2pt}
    \caption{
    Evaluation curves under in-space harness configurations on ClawGym-Bench.
    }
    \label{fig:eval_curve_whitebox}
    \vspace{-8pt}
\end{wrapfigure}

\noindent $\bullet$ \emph{Consistency across Model Scales and RL Algorithms.}
Harness scaling yields consistent gains across both backbone sizes and optimization methods (Table~\ref{tab:main_results}). For Qwen3-4B-Instruct, the overall average improves from 29.81 to 39.04 with GRPO and from 33.91 to 36.93 with PPO. Similar improvements are observed for Qwen3-30B-A3B-Instruct, from 37.29 to 42.78 with GRPO and from 36.31 to 41.89 with PPO. The improvement therefore persists across both model scale
and RL algorithm rather than being specific to a single training setting.
Additional out-of-space evaluation on OpenClaw shows the same trend in
Appendix~\ref{app:evaluation}.


\noindent $\bullet$ \emph{Transfer to External Black-Box Harnesses.}
Harness Scaling further transfers to the external black-box
harness. 
As shown in
Table~\ref{tab:blackbox_comparison}, it achieves average Pass@1 scores of
37.22 and 41.20 with Qwen3-4B-Instruct and Qwen3-30B-A3B-Instruct,
respectively, exceeding black-box SFT (35.08 and 38.08) and black-box RL
(33.38 and 36.28). Both black-box baselines are trained through OpenClaw
and transferred to Claude Code without further adaptation. This comparison
shows that structured training over diverse white-box harnesses can yield
stronger cross-harness transfer than direct training through a single
complex black-box harness.

\subsection{Ablation Studies}

%








\begin{wraptable}{r}{0.48\columnwidth}
\vspace{-6pt}
\centering
\caption{
Ablation of structured selection and assignment on Qwen3-4B-Instruct.
Each entry reports In-Space / Claude Code.
}
\label{tab:ablation}

\scriptsize
\setlength{\tabcolsep}{2.5pt}
\renewcommand{\arraystretch}{1.08}

\begin{tabular}{@{}lcc@{}}
\toprule
\textbf{Setting}
& \textbf{ClawGym-Bench}
& \textbf{PinchBench} \\
\midrule
Random-18
& 47.90 / 42.78
& 51.28 / 48.74 \\

Random Assign.
& 50.46 / 43.17
& 53.20 / 50.23 \\

\textbf{Ours}
& \textbf{52.78 / 45.37}
& \textbf{56.24 / 51.47} \\
\bottomrule
\end{tabular}

\vspace{-6pt}
\end{wraptable}

We further examine the key design choices in Harness Scaling, including
the construction of the training harness set, the number of harnesses
assigned to each task, and the locality-aware assignment strategy.
For the training harness set, we compare a single harness, three harnesses
providing first-order coverage over all component realizations, our
structured 18-harness set, a random 18-harness set, and an exhaustive
setting using all 243 configurations. Unless otherwise specified,
multi-harness settings use \(K=2\) with the same locality-aware assignment
and training budget. The full-space setting includes the white-box
evaluation configurations and is therefore used as a fixed-budget stress
test rather than a held-out generalization setting. We report the effects
of the training harness set and per-task harness diversity in
Figure~\ref{fig:ablation}. Table~\ref{tab:ablation} further compares our
structured 18-harness set with a random set of the same size, and
locality-aware assignment with random assignment.

\noindent $\bullet$ \emph{Benefits Depend on Harness Selection.}
Moving beyond a single fixed harness substantially improves generalization,
but performance does not increase monotonically with the number of training
configurations. Using three harnesses with first-order component coverage
yields gains of \(+11.43/+4.84\) points on ClawGym-Bench and
\(+3.44/+7.86\) points on PinchBench under in-space/black-box evaluation.
Our structured 18-harness set further improves over the three-harness
setting by \(+4.02/+1.57\) and \(+7.20/+2.66\) points, respectively,
while also outperforming Random-18 and the full 243-harness setting under
the same training budget. Notably, the full-space setting is trained on
all harness configurations, including those used for white-box evaluation,
yet still performs worse than the structured 18-harness set. This suggests
that exhaustive coverage alone does not guarantee better performance under
a fixed optimization budget, and that how harness diversity is organized
matters in addition to configuration count.

\noindent $\bullet$ \emph{Effect of Per-Task Harness Diversity.}
With \(K=1\), each task is associated with only one harness, while
different tasks remain distributed across the selected harness set.
As shown in Figure~\ref{fig:ablation}(c--d), this setting already yields
clear gains over Single-Harness RL, indicating that harness diversity
across tasks provides substantial benefits even without exposing the same
task to multiple harnesses.
Increasing \(K\) from 1 to 2 further improves performance. We additionally
evaluate a matched Repeat-2 control that follows the same two-step task
schedule while reusing the same sampled harness. Repeat-2 shows mixed
changes relative to \(K=1\), while \(K=2\) consistently performs better
than Repeat-2 across both benchmarks and evaluation settings. This suggests
that repeated task exposure alone does not fully account for the improvement
from \(K=2\), and that varying the harness for the same task provides
additional benefit. Increasing \(K\) to 4 provides no further improvement,
with \(K=2\) achieving the best overall performance among the evaluated
settings. Detailed results are provided in
Appendix~\ref{app:addition_abla}.


\noindent $\bullet$ \emph{Locality-Aware Assignment Complements Global Coverage.}
Replacing locality-aware assignment with random assignment reduces
performance by \(2.32/2.20\) points on ClawGym-Bench and
\(3.04/1.24\) points on PinchBench under in-space/black-box evaluation.
This pattern supports the global--local design of Harness Scaling, where
random anchor sampling maintains broad coverage across the training harness
space, while nearest-harness assignment provides controlled local variation
for each task.

\begin{figure*}[t]
    \centering

    \begin{subfigure}[t]{0.24\textwidth}
        \centering
        \includegraphics[width=\linewidth]{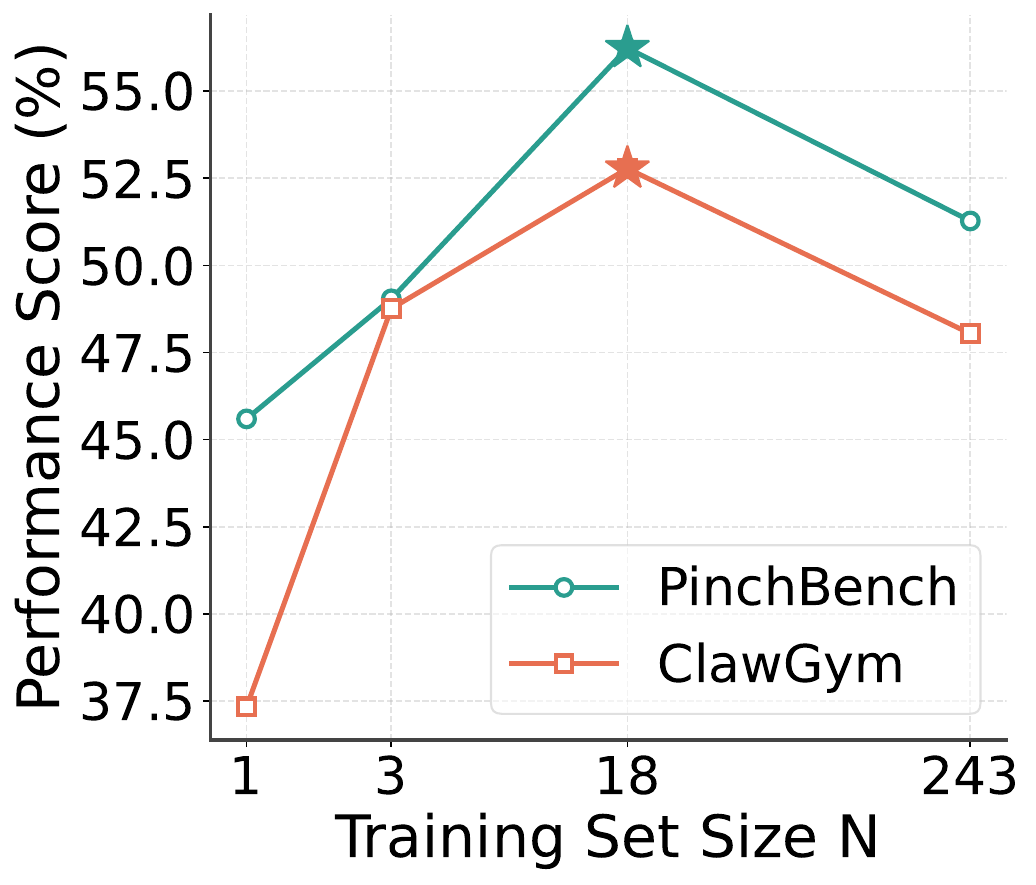}
        \caption{White-Box}
        \label{fig:ablation_a}
    \end{subfigure}
    \hfill
    \begin{subfigure}[t]{0.24\textwidth}
        \centering
        \includegraphics[width=\linewidth]{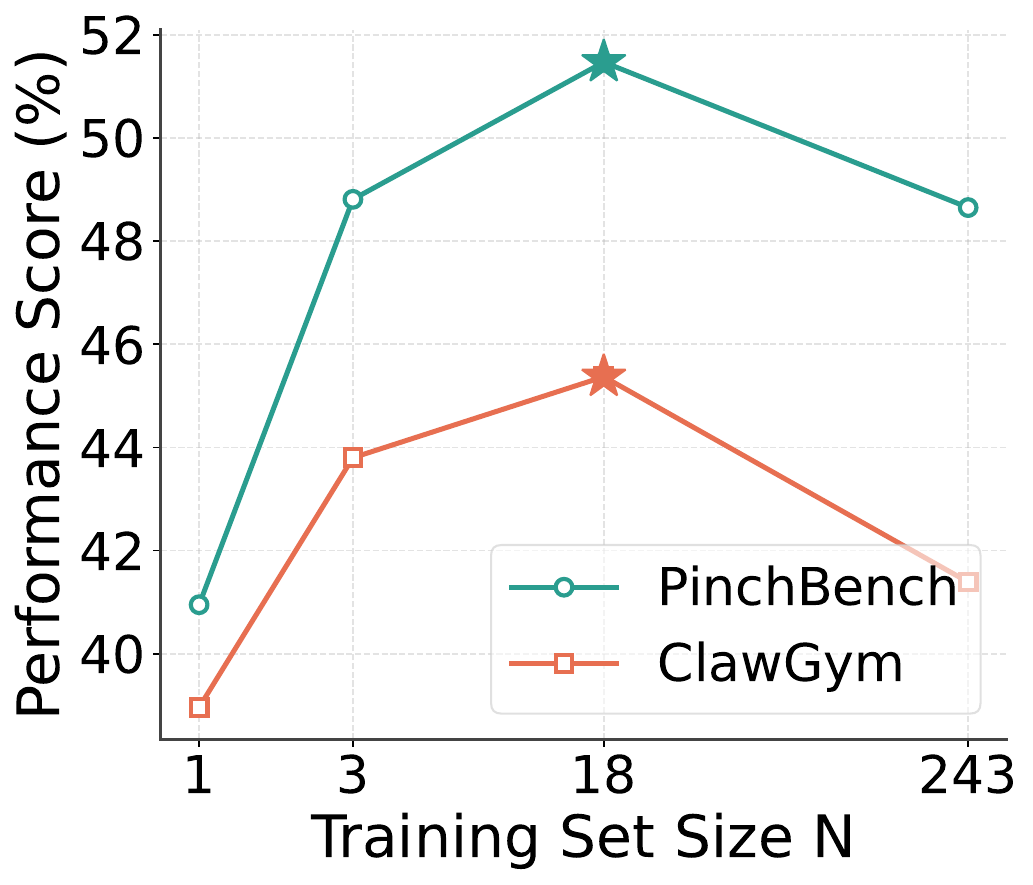}
        \caption{Black-Box}
        \label{fig:ablation_b}
    \end{subfigure}
    \hfill
    \begin{subfigure}[t]{0.24\textwidth}
        \centering
        \includegraphics[width=\linewidth]{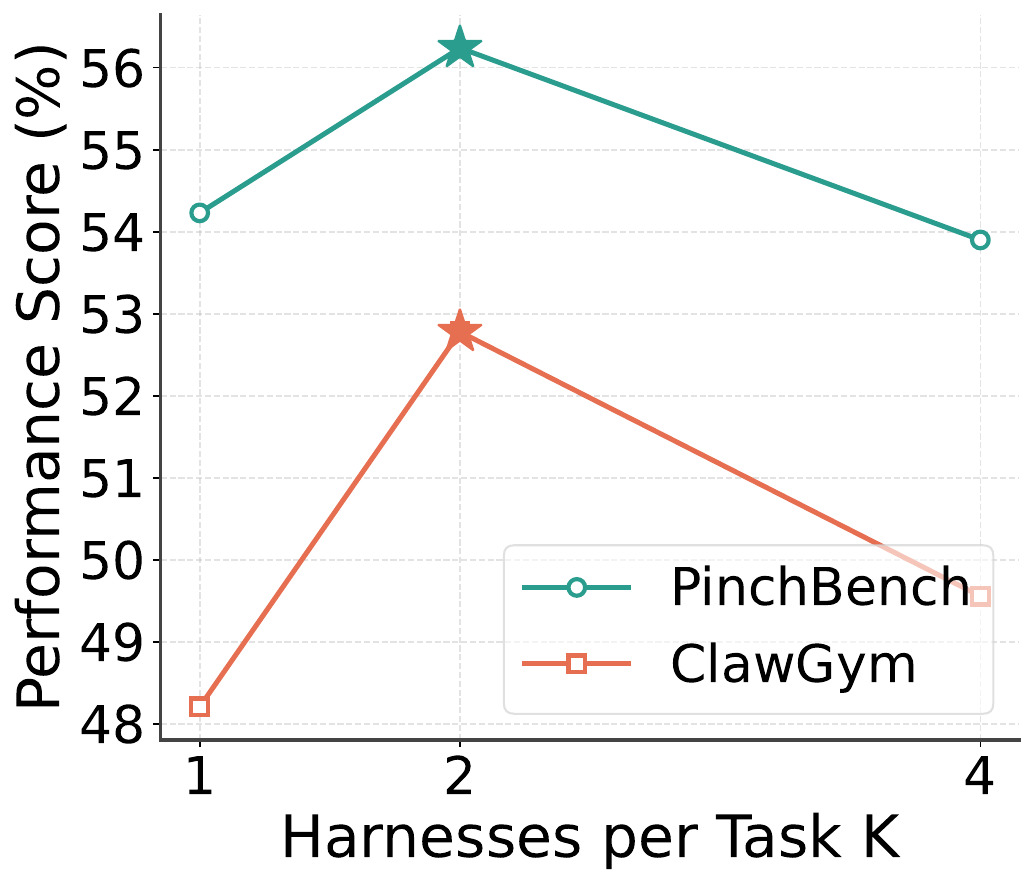}
        \caption{White-Box}
        \label{fig:ablation_c}
    \end{subfigure}
    \hfill
    \begin{subfigure}[t]{0.24\textwidth}
        \centering
        \includegraphics[width=\linewidth]{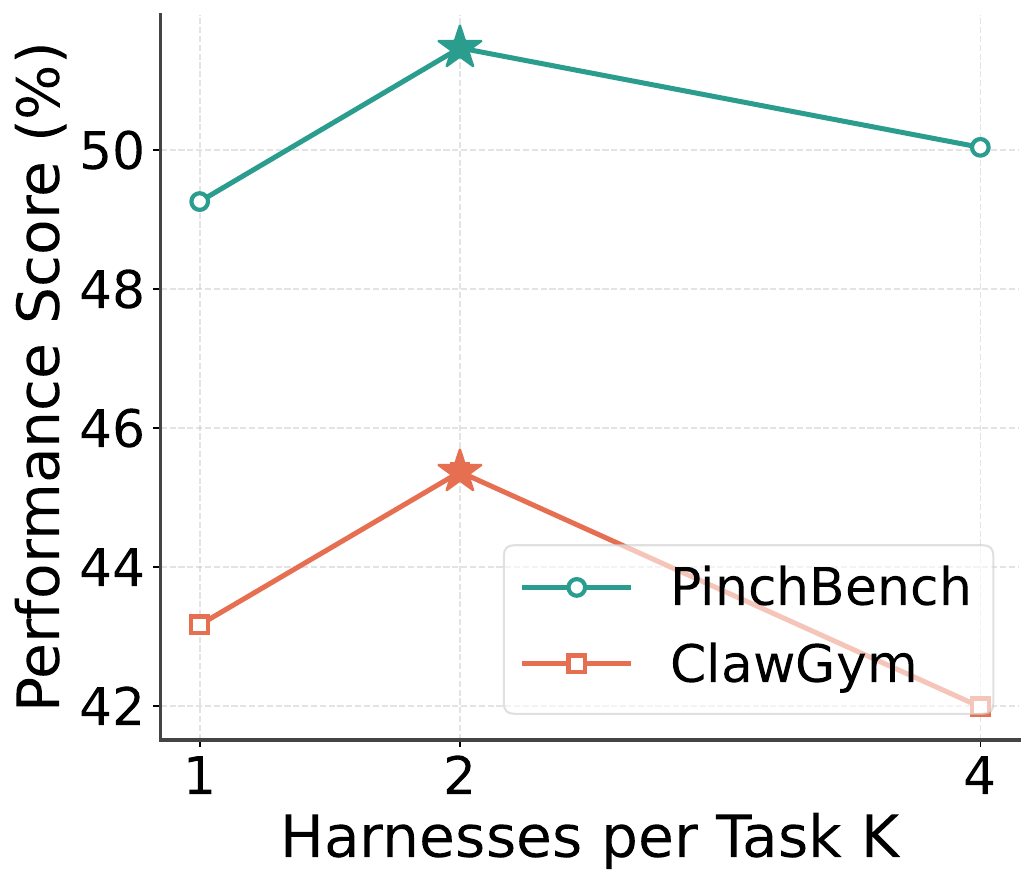}
        \caption{Black-Box}
        \label{fig:ablation_d}
    \end{subfigure}

\caption{
Ablation studies of key design choices in Harness Scaling on
Qwen3-4B-Instruct.
(a--b) Performance across training harness sets of different sizes under
white-box and external black-box~(Claude Code) evaluation. The full 243-harness setting
includes the white-box evaluation configurations and serves as a
fixed-budget stress test rather than a held-out generalization setting.
(c--d) Effect of the number of harnesses assigned per task during training
under the same two evaluation settings.
}

    \label{fig:ablation}
\end{figure*}
\section{Analysis}

We further investigate behavioral patterns associated with cross-harness
generalization. Specifically, we examine whether Harness Scaling is
associated with stronger execution capabilities and more robust performance
as the execution configuration changes. For each model scale and approach,
we analyze 200 complete trajectories per evaluated harness, generated from
the same set of ClawGym-Bench tasks to enable direct comparisons across
approaches and harnesses. More details are provided in
Appendix~\ref{app:analysis_details}.



\subsection{Harness-Related Capability Analysis}

Harness Scaling strengthens capabilities that are directly relevant to
operating under different harness configurations.
Figures~\ref{fig:analysis_a} and~\ref{fig:analysis_b} compare
Single-Harness RL and Harness Scaling along six behavioral dimensions:
instruction following, interface handling, observation reading, state
tracking, workflow compliance, and error recovery.

The largest gains appear in state tracking, workflow compliance, and error
recovery, indicating better maintenance of task progress, adherence to
execution procedures, and recovery from failures. Instruction following
and interface handling also improve, while the gain in observation reading
is relatively smaller. These results suggest that training across harnesses
particularly strengthens capabilities needed to sustain effective execution
as the surrounding harness changes.

\subsection{Cross-Harness Robustness and Specialization}

Harness Scaling makes task performance less sensitive to changes in the
execution configuration. Figures~\ref{fig:analysis_c}
and~\ref{fig:analysis_d} compare paired rewards for the same tasks under
two fixed held-out in-space harnesses: Harness A
(\texttt{S0T2O1C1W2}) and Harness B (\texttt{S2T1O1C2W0}). The two
configurations differ in four harness components
(\(d_{\mathrm H}=4\)) and are used consistently for both model scales and training methods. Each point represents one task, and the diagonal
denotes equal reward under the two configurations.
Under Harness Scaling, the reward distributions shift toward the
upper-right region for both model scales, indicating stronger performance
under both harnesses. The paired rewards also appear more closely aligned
across configurations, suggesting greater consistency as the harness
changes.

\begin{figure*}[t]
    \centering

    \begin{subfigure}[t]{0.24\textwidth}
        \centering
        \includegraphics[width=\linewidth]{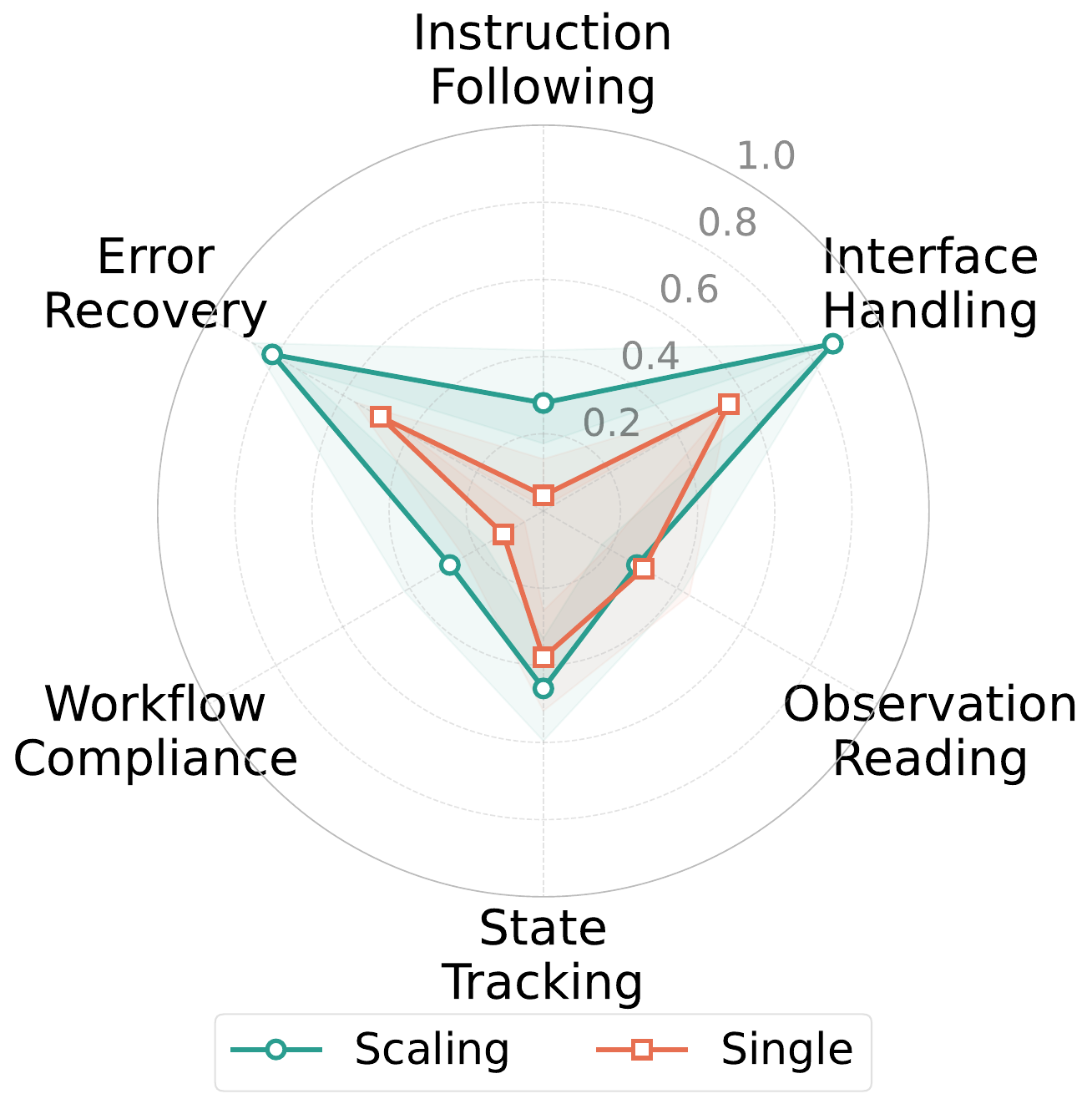}
        \caption{Qwen3-4B}
        \label{fig:analysis_a}
    \end{subfigure}
    \hfill
    \begin{subfigure}[t]{0.24\textwidth}
        \centering
        \includegraphics[width=\linewidth]{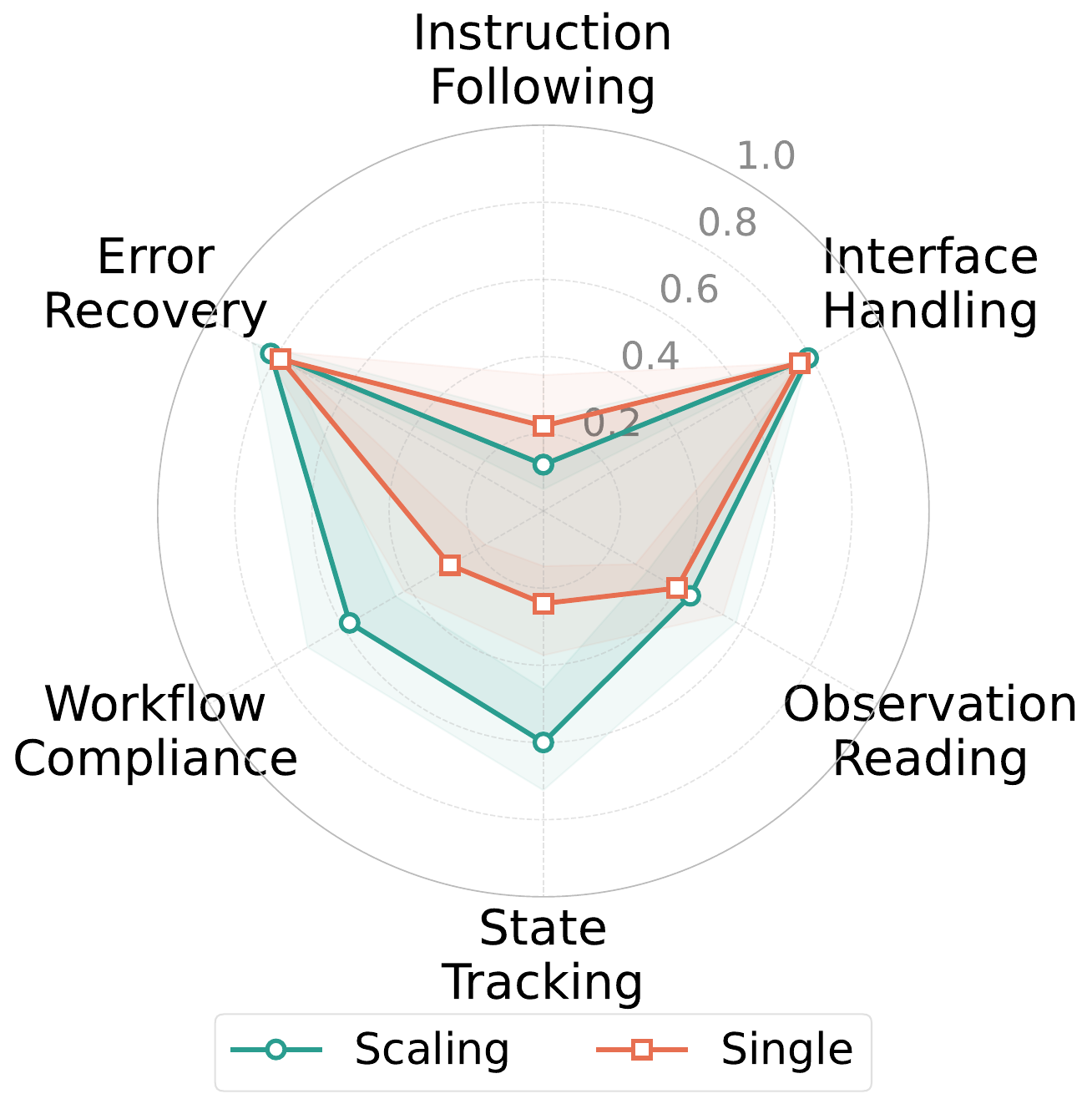}
        \caption{Qwen3-30B-A3B}
        \label{fig:analysis_b}
    \end{subfigure}
    \hfill
    \begin{subfigure}[t]{0.24\textwidth}
        \centering
        \includegraphics[width=\linewidth]{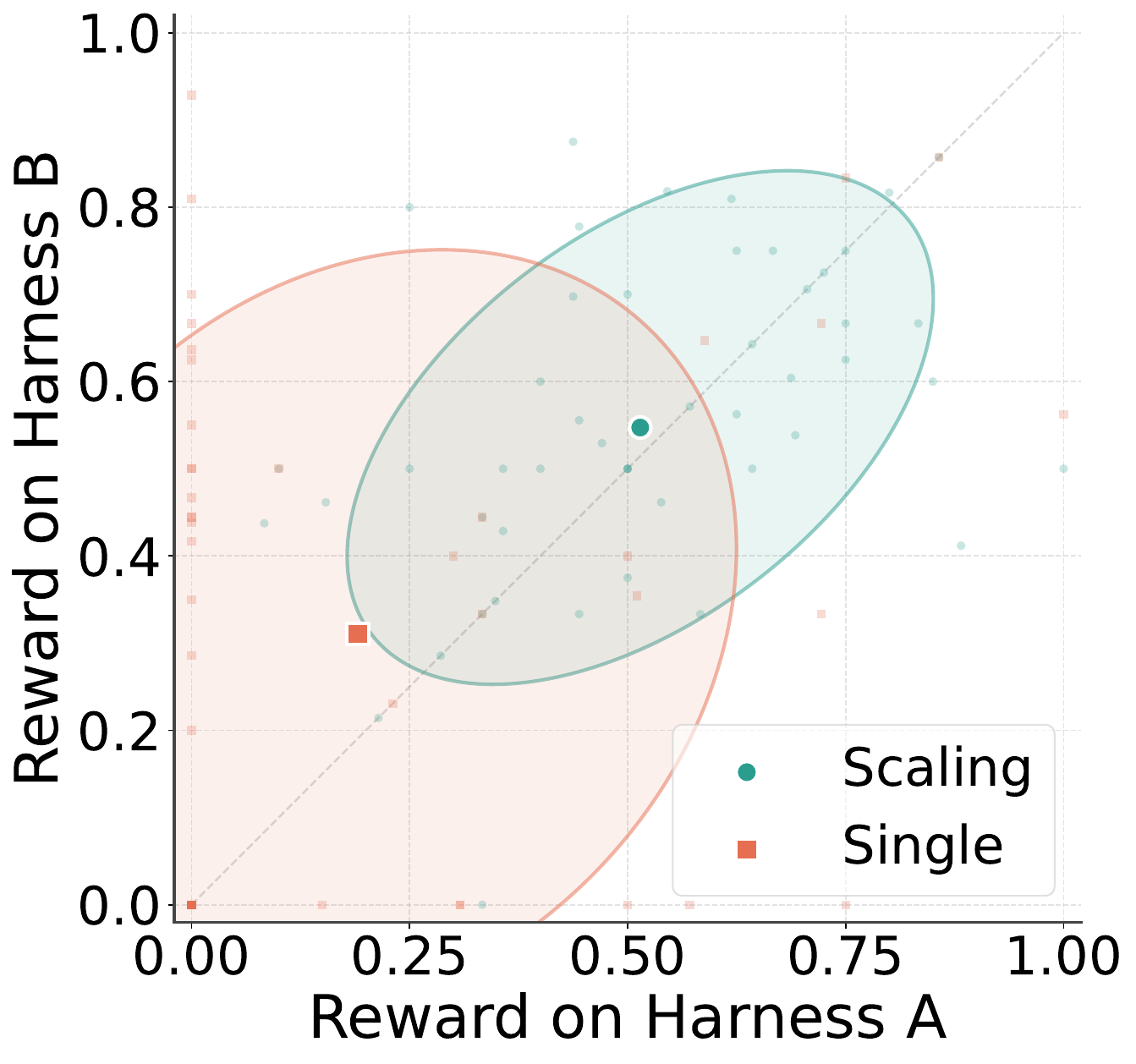}
        \caption{Qwen3-4B}
        \label{fig:analysis_c}
    \end{subfigure}
    \hfill
    \begin{subfigure}[t]{0.24\textwidth}
        \centering
        \includegraphics[width=\linewidth]{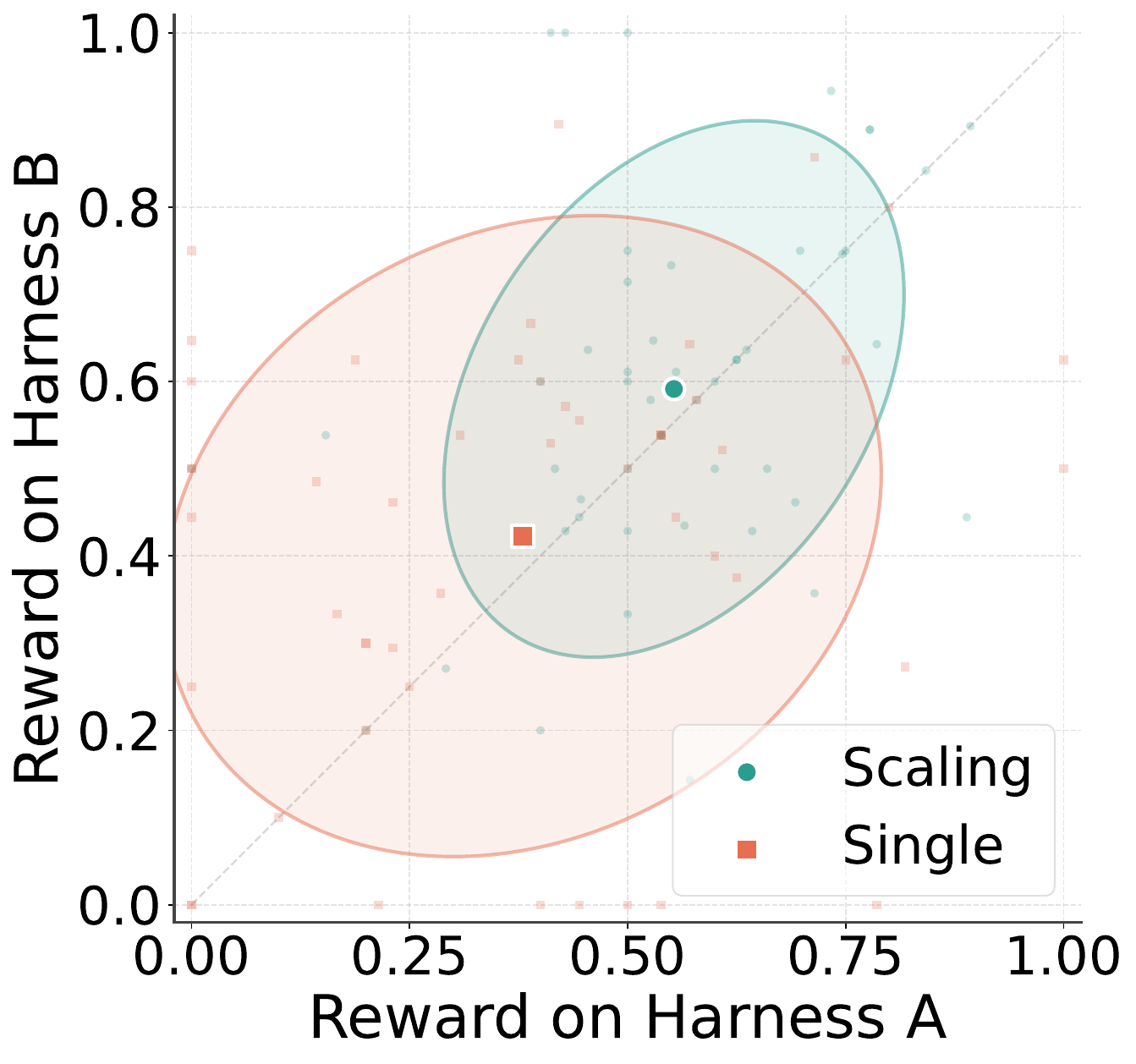}
        \caption{Qwen3-30B-A3B}
        \label{fig:analysis_d}
    \end{subfigure}

\caption{
Behavioral analysis comparing Single-Harness RL and Harness Scaling.
\textbf{(a--b)} Capability profiles under in-space evaluation.
\textbf{(c--d)} Cross-harness reward consistency on the same tasks under
two fixed held-out harnesses, Harness A (\texttt{S0T2O1C1W2}) and Harness B
(\texttt{S2T1O1C2W0}); ellipses show 1.5$\sigma$ covariance regions and
large markers denote group means.
}

    \label{fig:analysis}
\end{figure*}





\begin{wraptable}{r}{0.42\columnwidth}
\vspace{-8pt}
\centering
\caption{
Base-harness vs.\ cross-harness performance for
Qwen3-30B-A3B-Instruct with GRPO on ClawGym-Bench.
}
\label{tab:seen_harness}

\scriptsize
\setlength{\tabcolsep}{1.5pt}
\renewcommand{\arraystretch}{1.05}

\begin{tabular}{@{}lccc@{}}
\toprule
& \textbf{Base Harness}
& \textbf{In-Space}
& \textbf{Claude Code} \\
\midrule
Single
& 59.82
& 48.12
& 46.02 \\
Scaling
& 59.63
& 58.60
& 49.62 \\
\midrule
$\Delta$
& $-0.19$
& \textbf{$+10.48$}
& \textbf{$+3.60$} \\
\bottomrule
\end{tabular}

\vspace{-7pt}
\end{wraptable}

We further examine whether improved cross-harness generalization comes at
the cost of performance on the base training harness. As shown in
Table~\ref{tab:seen_harness}, Harness Scaling incurs only a 0.19-point
decrease on the base harness, while improving the ClawGym-Bench in-space
score by 10.48 points and Claude Code performance by 3.60 points. The large
cross-harness gains relative to the negligible change on the base harness
suggest that Harness Scaling improves robustness to execution variation
while largely preserving performance under the original training
configuration.
\section{Conclusion}


In this work, we systematically study Harness Scaling as a distinct scaling dimension for agentic reinforcement learning, focusing on how harness diversity can be organized to improve cross-harness generalization. We formalize harness diversity as a structured configuration space and develop a framework that integrates structured harness selection with cross-harness policy training. Across benchmarks, model scales, and RL algorithms, Harness Scaling substantially improves generalization to unseen white-box configurations and transfers effectively to unseen external black-box harnesses. Our results further show that the benefits of scaling depend on how harness diversity is selected and incorporated into training, rather than on diversity alone. More broadly, these findings expand the scope of agentic RL scaling beyond what agents solve and where they operate to how they interact and execute, positioning harness variation as a complementary dimension for improving generalization across diverse execution systems.


\bibliography{ref}
\newpage
\appendix

\section{Harness Component Variants}
\label{app:harness_variants}


We instantiate the factorized harness space through a model-driven
transformation procedure. Starting from the base harness, Opus-4.7 generates
alternative realizations for each of the five components: system prompt
(\(S\)), tool interface (\(T\)), observation (\(O\)), context management
(\(C\)), and workflow (\(W\)). Each transformation targets a single
component and introduces a specified form of variation while preserving the
underlying task, environment, and functionality of the remaining harness
components.

The transformation procedure explicitly constrains non-target components to
remain unchanged in task semantics, environment behavior, and functionality,
so that each resulting realization primarily reflects variation along the
designated component dimension.

This procedure yields three realizations per component and
\(3^5=243\) complete harness configurations.
Table~\ref{tab:harness_variants} summarizes the resulting variants.

\begin{table*}[t]
\small
\centering
\caption{
Harness component realizations instantiated by the model-driven
transformation procedure for the \(3^5=243\) factorized harness space.
}
\label{tab:harness_variants}

\footnotesize
\setlength{\tabcolsep}{4.2pt}
\renewcommand{\arraystretch}{1.12}

\begin{tabularx}{0.9\textwidth}{
@{}
p{0.045\textwidth}
p{0.11\textwidth}
YYY
@{}
}
\toprule
\textbf{Dim.}
& \textbf{Aspect}
& \textbf{Base}
& \textbf{Variant I}
& \textbf{Variant II} \\
\midrule

\(S\)
& System prompt
& Direct system instruction specifying the
agent role and execution rules.
& The same requirements are embedded in a simulated product context, with additional platform and UI descriptions.
& The requirements are reframed through a narrative and role-based setting. \\

\addlinespace[2pt]

\(T\)
& Tool interface
& Original tool names and schemas, such as
\texttt{execute\_bash} and \texttt{file\_editor}, without distractors.
& Functional tools are exposed through semantically related aliases,
together with additional non-functional distractor tools.
& Functional tools are exposed through opaque identifiers and augmented
with distractor tools, largely removing lexical cues to tool functionality. \\

\addlinespace[2pt]

\(O\)
& Observation
& Raw tool outputs without additional prefixes or suffixes.
& Outputs are augmented with simulated infrastructure metadata, including
pod names, namespaces, and scheduling information.
& Outputs are augmented with structured XML-style metadata, including
tool name, exit code, process ID, duration, and execution node. \\

\addlinespace[2pt]

\(C\)
& Context management
& The complete interaction history remains available in the active context.
& LLM-based compaction periodically summarizes earlier history while
retaining the five most recent turns in context.
& A more aggressive compaction policy retains the three most recent turns,
with summarized history externalized to an auxiliary file. \\

\addlinespace[2pt]

\(W\)
& Workflow
& Standard free-form execution without additional procedural constraints.
& A plan-first workflow requires explicit planning before task execution.
& Planning is combined with periodic progress verification every three
steps. \\

\bottomrule
\end{tabularx}
\end{table*}

\paragraph{System Prompt (\(S\)).}
This dimension varies how the same operational requirements are presented
to the policy. The transformed realizations range from direct instructions
to product-oriented and narrative formulations while preserving the
underlying execution requirements.

\paragraph{Tool Interface (\(T\)).}
The base interface provides five functional tool families: file
reading/writing/editing, Bash execution, code interpretation, search, and
web fetching. The transformation varies how these tools are exposed while
preserving their underlying functionality. Semantic aliases retain partial
cues to tool usage, whereas opaque identifiers substantially reduce such
cues; distractor tools further increase interface complexity.

\paragraph{Observation (\(O\)).}
This dimension varies the representation of tool feedback. Simulated
infrastructure or structured execution metadata is introduced around the
original output, while the task-relevant tool result remains unchanged.

\paragraph{Context Management (\(C\)).}
This dimension varies how interaction history remains accessible during
long-horizon execution, ranging from full retention to in-context
summarization and more aggressive externalized compaction. Further details
on how compaction participates in policy training are provided in
Appendix~\ref{app:training_evaluation}.

\paragraph{Workflow (\(W\)).}
This dimension varies the procedural structure of task execution, ranging
from free-form interaction to explicit planning and periodic progress
verification while preserving the underlying task and available actions.

\section{Additional Details on Structured Harness Selection}
\label{app:selection_details}

\subsection{Orthogonal Design and Coverage}

We provide additional details on the structured harness selection in
Section~\ref{sec:Harness Set Construction}. Recall that the harness space
contains \(M\) component dimensions, each with \(L\) realizations. We
construct the training harness set using a strength-2 orthogonal design,
under which every realization pair from any two component dimensions occurs
equally often.

\begin{proposition}[Coverage of a Strength-2 Orthogonal Design]
\label{prop:oa_coverage}
Consider a strength-2 orthogonal design over \(M\) component dimensions,
each with \(L\) realizations. If every pairwise realization combination
occurs exactly \(\lambda\) times, then every realization of each component
appears exactly \(\lambda L\) times, and the total number of selected
configurations satisfies
\begin{equation}
N = \lambda L^2.
\end{equation}
\end{proposition}

\begin{proof}
Consider any two component dimensions. By the strength-2 property, each of
their \(L^2\) realization pairs occurs exactly \(\lambda\) times. Fixing one
realization of the first component and summing over all \(L\) realizations
of the second shows that the fixed realization appears exactly
\(\lambda L\) times. Therefore, every realization within each component is
represented equally often.

Moreover, the \(L^2\) realization pairs between any two components
partition the selected configurations. Since each pair occurs
\(\lambda\) times, the total number of configurations is
\[
N = \lambda L^2.
\]
\end{proof}

A strength-2 orthogonal design additionally satisfies the classical Rao
bound~\cite{rao1947factorial,hedayat1999orthogonal},
\begin{equation}
N \geq 1 + M(L-1).
\label{eq:rao_bound_appendix}
\end{equation}
The bound provides a lower bound on the number of configurations required
for a strength-2 orthogonal design.

In our harness space, we use \(M=5\) component dimensions with
\(L=3\) realizations per component. Proposition~\ref{prop:oa_coverage}
requires
\begin{equation}
N = 9\lambda,
\end{equation}
while the Rao bound gives
\begin{equation}
N \geq 1 + 5(3-1) = 11.
\end{equation}
Since \(N\) must be a multiple of \(9\), the smallest design size permitted
by these constraints is \(N=18\). We instantiate a corresponding
strength-2 orthogonal design with \(\lambda=2\). Under this design, each
realization of every harness component appears six times, while every
realization pair across any two components appears exactly twice.

The resulting 18-harness set provides balanced first-order and pairwise
coverage of the full factorial space while leaving many higher-order
configurations unseen during training. This provides a compact training set
for Harness Scaling while retaining held-out combinations for evaluating
in-space combinatorial generalization.

\subsection{Training Harness Set}

The resulting structured training set consists of the 18 harness
configurations listed in Table~\ref{tab:training_harness_set}. 

\begin{table}[h]
\centering
\caption{The 18 configurations in the structured training harness set
\(\mathcal{H}_{\mathrm{tr}}\).}
\label{tab:training_harness_set}
\scriptsize
\setlength{\tabcolsep}{6pt}
\renewcommand{\arraystretch}{1.03}
\begin{tabular}{lll}
\toprule
\texttt{S0T0O0C0W0} & \texttt{S0T1O1C1W1} & \texttt{S0T2O2C2W2} \\
\texttt{S1T0O0C1W1} & \texttt{S1T1O1C2W2} & \texttt{S1T2O2C0W0} \\
\texttt{S2T0O1C0W2} & \texttt{S2T1O2C1W0} & \texttt{S2T2O0C2W1} \\
\texttt{S0T0O2C2W1} & \texttt{S0T1O0C0W2} & \texttt{S0T2O1C1W0} \\
\texttt{S1T0O1C2W0} & \texttt{S1T1O2C0W1} & \texttt{S1T2O0C1W2} \\
\texttt{S2T0O2C1W2} & \texttt{S2T1O0C2W0} & \texttt{S2T2O1C0W1} \\
\bottomrule
\end{tabular}
\end{table}

\section{Implementation Details}
\label{app:training_evaluation}

\subsection{Training}
\label{app:training}
\paragraph{Training Configuration.}
We train all models with veRL~\cite{Sheng_2025} using asynchronous vLLM~\cite{kwon2023efficientmemorymanagementlarge} rollouts on NVIDIA
H20 GPUs, using 16 GPUs for Qwen3-4B-Instruct and 32 GPUs for Qwen3-30B-A3B-Instruct.
Table~\ref{tab:training_hparams} summarizes the main training
hyperparameters. For GRPO, each update uses 16 tasks with eight rollouts
per task, while PPO uses 128 tasks with one rollout per task. We use an
actor learning rate of \(1\times10^{-6}\), asymmetric clipping ratios of
0.2 and 0.28, and no KL or entropy regularization. Dynamic batching is
enabled throughout training. For PPO, we use GAE with
\(\gamma=\lambda=1\); the critic masks observation tokens from value
regression, and its value head is initialized to zero.

\begin{table}[t]
\small
\centering
\caption{
Main training hyperparameters for GRPO and PPO experiments.
}
\label{tab:training_hparams}

\footnotesize
\setlength{\tabcolsep}{4.2pt}
\renewcommand{\arraystretch}{1.12}

\begin{tabular*}{0.9\textwidth}{
@{\extracolsep{\fill}}lcccc@{}
}
\toprule
\textbf{Parameter}
& \textbf{4B GRPO}
& \textbf{4B PPO}
& \textbf{30B-A3B GRPO}
& \textbf{30B-A3B PPO} \\
\midrule

Prompt batch size
& 16 & 128 & 16 & 128 \\

Rollouts per prompt
& 8 & 1 & 8 & 1 \\

Actor mini-batch size
& 16 & 128 & 16 & 128 \\

Actor learning rate
& \(1\times10^{-6}\)
& \(1\times10^{-6}\)
& \(1\times10^{-6}\)
& \(1\times10^{-6}\) \\

Max. prompt length
& 8K & 8K & 4K & 4K \\

Max. response length
& 32K & 32K & 60K & 60K \\

Actor clip ratio (low/high)
& 0.20 / 0.28
& 0.20 / 0.28
& 0.20 / 0.28
& 0.20 / 0.28 \\

GAE \(\gamma / \lambda\)
& -- & 1.0 / 1.0
& -- & 1.0 / 1.0 \\

Critic learning rate
& -- & \(1\times10^{-5}\)
& -- & \(1\times10^{-5}\) \\

Value clipping
& -- & 0.5
& -- & 1.0 \\

Training GPUs
& 16 & 16 & 32 & 32 \\

\bottomrule
\end{tabular*}
\end{table}


\paragraph{Training with Context Compaction.}
Context compaction can split a single rollout into multiple training
segments. All segments originating from the same rollout share the same
rollout-level reward and participate in policy optimization, causing the
number of training segments in a batch to vary dynamically. Segment
expansion does not alter the rollout-level grouping used by GRPO: for a
group of eight original rollouts, reward normalization is always computed
over the eight rollout rewards, regardless of how many segments each
rollout produces. Segments inherit the group identity and normalized
advantage of their parent rollout and do not introduce additional entries
when computing the group mean or standard deviation.
Policy loss is aggregated with token-level mean reduction over all valid
policy-generated tokens in the batch. Each valid token therefore
contributes once to the loss regardless of how its parent rollout is
segmented, so splitting a rollout into multiple segments does not introduce
additional sample-level weight. This preserves both the rollout-level
reward semantics and consistent token-level weighting while allowing all
valid post-compaction segments to participate in training.


\paragraph{Black-Box Baselines.}
For the black-box baselines, we use OpenClaw v2026.4.5 as the training
harness and reserve Claude Code v2.1.205 as an unseen external harness
for evaluation. Black-Box SFT and Black-Box RL represent two complementary
adaptation strategies to the same single black-box harness: the former
fine-tunes the policy on 12,800 trajectories collected on ClawGym-Task
through OpenClaw for three epochs, while the latter performs online policy
optimization through OpenClaw on the same training set for 300 updates
with 128 rollouts per update. This gives both settings approximately
38,400 trajectory-level training exposures, although SFT reuses a fixed
trajectory set across epochs whereas RL collects trajectories online.
For the comparison in Table~\ref{tab:blackbox_comparison}, Black-Box RL
uses GRPO, and Harness Scaling uses the corresponding GRPO-trained
checkpoint. Both RL settings use the same backbone models, matched
optimization budgets, the same main GRPO hyperparameters, and the same
evaluation protocol.

\subsection{Evaluation}
\label{app:evaluation}


\paragraph{In-Space Aggregation.}
For compact notation, indices \(0\), \(1\), and \(2\) denote the base
realization, Variant I, and Variant II of each harness component,
respectively. We evaluate each checkpoint on the 18 held-out harness
configurations listed in Table~\ref{tab:heldout_harness_set}.
\begin{table}[h]
\centering
\caption{
The 18 held-out configurations used for in-space evaluation.
}
\label{tab:heldout_harness_set}
\scriptsize
\setlength{\tabcolsep}{5pt}
\renewcommand{\arraystretch}{1.03}
\begin{tabular}{lll}
\toprule
\texttt{S0T2O1C1W2} & \texttt{S2T1O1C2W0} & \texttt{S0T1O0C2W2} \\
\texttt{S2T2O0C0W1} & \texttt{S0T1O2C0W1} & \texttt{S1T0O0C2W1} \\
\texttt{S0T0O0C1W2} & \texttt{S0T0O1C0W1} & \texttt{S0T2O2C2W0} \\
\texttt{S1T0O1C1W0} & \texttt{S1T0O2C2W2} & \texttt{S1T1O0C0W0} \\
\texttt{S1T2O1C0W2} & \texttt{S1T2O2C1W1} & \texttt{S2T0O2C0W0} \\
\texttt{S2T1O1C1W2} & \texttt{S2T1O2C2W1} & \texttt{S2T2O0C1W0} \\
\bottomrule
\end{tabular}
\end{table}
The 18 evaluation configurations are fixed across benchmarks, model
scales, and RL algorithms, and are excluded from training in all main
experiments and ablations except the full 243-harness setting. In the
full-space setting, they are included in training by construction; its
white-box result is therefore interpreted as a fixed-budget comparison
rather than held-out generalization.
The evaluation set provides exact marginal balance across component
realizations: for every harness component, each of its three realizations
appears exactly six times. The selected configurations also span diverse
combinations within the factorized space, with pairwise Hamming distances
ranging from 2 to 5. The set is selected solely based on harness
configurations and does not use model performance. We compute the benchmark
score separately under each of the 18 configurations and average the scores
to obtain the reported white-box evaluation performance. For settings in
which these configurations are excluded from training, this corresponds to
the reported In-Space performance.

\paragraph{Scoring.}
We report Pass@1 averaged over three independent evaluation runs. For each
run, every task is executed once under the corresponding harness setting,
and the resulting task-level scores are aggregated following the benchmark
evaluation protocol. The reported score is then obtained by averaging the
Pass@1 results across the three runs. For rubric-based evaluation requiring
a model judge, we use GPT-5.4 with the same judging protocol across all
methods and harness settings.

\paragraph{Additional Out-of-Space Evaluation.}
We further evaluate the white-box-trained policies on OpenClaw as
an additional external black-box harness. OpenClaw is not included
in the factorized harness space or used during the training of
Single-Harness RL and Harness Scaling, providing a complementary
out-of-space evaluation beyond Claude Code. We report results using
the GRPO-trained checkpoints, consistent with the black-box comparison in
Table~\ref{tab:blackbox_comparison}. As shown in
Table~\ref{tab:openclaw_eval}, Harness Scaling consistently improves
performance over Single-Harness RL across both model scales and all three
benchmarks.

\begin{table}[t]
\centering
\caption{
Out-of-space performance on OpenClaw as an additional external
black-box harness. OpenClaw is not used during the training of
Single-Harness RL or Harness Scaling. All results use GRPO-trained
checkpoints.
}
\label{tab:openclaw_eval}

\small
\setlength{\tabcolsep}{5pt}
\renewcommand{\arraystretch}{1.08}

\begin{tabular}{lcccc}
\toprule
\textbf{Method}
& \textbf{ClawGym-Bench}
& \textbf{JobBench-Easy}
& \textbf{PinchBench}
& \textbf{Avg.} \\
\midrule

\multicolumn{5}{l}{\textit{Qwen3-4B-Instruct}} \\
Single-Harness RL
& 43.18 & 7.67 & 45.20 & 32.02 \\
Harness Scaling
& \textbf{46.84} & \textbf{11.25} & \textbf{50.11} & \textbf{36.07} \\

\midrule

\multicolumn{5}{l}{\textit{Qwen3-30B-A3B-Instruct}} \\
Single-Harness RL
& 45.31 & 11.90 & 52.78 & 36.66 \\
Harness Scaling
& \textbf{50.16} & \textbf{14.81} & \textbf{58.62} & \textbf{41.20} \\

\bottomrule
\end{tabular}
\end{table}

\subsection{Additional Ablation Details}
\label{app:addition_abla}

\paragraph{Per-Task Harness Diversity.}
Table~\ref{tab:k_ablation} reports the numerical results for the
per-task harness diversity study. We additionally include a Repeat-2
control that matches the two-step task schedule of \(K=2\) while using
the same sampled anchor harness for both successive training steps.
Relative to \(K=1\), Repeat-2 produces mixed changes across benchmarks,
suggesting that repeated exposure to the same task--harness pair is not
uniformly beneficial. In contrast, \(K=2\) consistently outperforms
Repeat-2 across both benchmarks and evaluation settings, with gains of
\(3.02/1.77\) points on ClawGym-Bench and \(5.90/2.83\) points on
PinchBench under in-space/Claude Code evaluation. These results provide
additional evidence that cross-harness variation for the same task
contributes beyond repeated task exposure alone.

\begin{table}[t]
\centering
\caption{
Effect of per-task harness diversity on Qwen3-4B-Instruct with GRPO.
\textit{Repeat-2} matches the two-step task schedule of \(K=2\) while
using the same sampled anchor harness for both steps. Each entry reports
In-Space / Claude Code performance.
}
\label{tab:k_ablation}

\small
\setlength{\tabcolsep}{5pt}
\renewcommand{\arraystretch}{1.08}

\begin{tabular}{lcc}
\toprule
\textbf{Setting}
& \textbf{ClawGym-Bench}
& \textbf{PinchBench} \\
\midrule

Single-Harness RL
& 37.33 / 38.96
& 45.60 / 40.95 \\

Harness Scaling, \(K=1\)
& 48.21 / 43.17
& 54.23 / 49.26 \\

Repeat-2, Same Anchor
& 49.76 / 43.60
& 50.34 / 48.64 \\

Harness Scaling, \(K=2\)
& \textbf{52.78} / \textbf{45.37}
& \textbf{56.24} / \textbf{51.47} \\

Harness Scaling, \(K=4\)
& 49.56 / 41.99
& 53.90 / 50.04 \\

\bottomrule
\end{tabular}
\end{table}

\paragraph{Sensitivity to the Single-Harness Configuration.}
Our main Single-Harness RL baseline uses the base harness
\(h_{\mathrm{base}}\). To examine sensitivity to this choice, we
additionally train three fixed-harness baselines selected from
\(\mathcal H_{\mathrm{tr}}\) using only the geometry of the factorized
harness space, without reference to evaluation performance.
Specifically, we consider
\(h_{\mathrm{base}}=(0,0,0,0,0)\),
\(h_1=(1,1,1,2,2)\),
\(h_2=(2,1,2,1,0)\), and
\(h_3=(2,2,1,0,1)\),
where the coordinates follow the order \((S,T,O,C,W)\).
Among four-harness subsets containing \(h_{\mathrm{base}}\), this is one
of the maximally separated choices under Hamming distance, with a minimum
pairwise distance of four. The selection therefore spans structurally
distinct training configurations while remaining independent of model
performance. As shown in Table~\ref{tab:single_harness_sensitivity},
Single-Harness RL varies with the fixed training harness, whereas Harness
Scaling remains stronger than all four evaluated single-harness settings
under both In-Space and Claude Code evaluation. This provides additional
evidence that the observed advantage is not specific to the base training
harness.

\begin{table}[t]
\centering
\caption{
Sensitivity of Single-Harness RL to the fixed training harness on
Qwen3-4B-Instruct with GRPO. Results are averaged across the three
benchmarks. Harness Scaling is included for reference.
}
\label{tab:single_harness_sensitivity}

\small
\setlength{\tabcolsep}{7pt}
\renewcommand{\arraystretch}{1.08}

\begin{tabular}{lccc}
\toprule
\textbf{Setting}
& \textbf{Harness}
& \textbf{In-Space}
& \textbf{Claude Code} \\
\midrule

Single-Harness RL
& \(h_{\mathrm{base}}\)
& 31.06 & 28.57 \\

Single-Harness RL
& \(h_1\)
& 35.79 & 36.11 \\

Single-Harness RL
& \(h_2\)
& 33.52 & 33.11 \\

Single-Harness RL
& \(h_3\)
& 35.15 & 32.72 \\

\midrule

Harness Scaling
& \(\mathcal H_{\mathrm{tr}}\)
& \textbf{40.86} & \textbf{37.22} \\

\bottomrule
\end{tabular}
\end{table}

\section{Additional Analysis Details}
\label{app:analysis_details}

\subsection{Behavioral Metrics}
We use six trajectory-level behavioral metrics to characterize
harness-related capabilities beyond task-level performance.
\textbf{Instruction Following} measures whether the policy correctly follows
the system instructions and execution constraints specified by the harness.
\textbf{Interface Handling} measures whether the policy correctly interprets
and uses the available tool interface, particularly when tool names or
schemas differ from the base harness.
\textbf{Observation Reading} measures whether the policy extracts
task-relevant information from tool feedback despite additional formatting
or distracting metadata.
\textbf{State Tracking} measures whether the policy maintains an accurate
understanding of task progress, intermediate results, and the evolving
workspace state throughout long-horizon execution.
\textbf{Workflow Compliance} measures whether the policy follows procedural
requirements imposed by the harness, such as explicit planning or periodic
verification.
\textbf{Error Recovery} measures whether the policy recognizes execution
failures, adapts its strategy, and continues toward task completion rather
than repeating or propagating the same error.

\subsection{Metric Evaluation}
We evaluate the six capabilities from complete agent trajectories using
GPT-5.4 as the trajectory judge. For each model scale and each of the
18 held-out in-space harnesses, we collect 200 complete trajectories for
both Single-Harness RL and Harness Scaling on the same 200 ClawGym-Bench
tasks. For each trajectory, the judge is provided with the task instruction,
harness specification, and complete interaction history, including policy
actions, tool calls, and tool observations.

Each capability is scored independently on a five-point scale. A score of
1 indicates that the corresponding capability is largely absent or repeatedly
violated; 2 indicates substantial deficiencies; 3 indicates partial or
inconsistent satisfaction; 4 indicates generally correct behavior with
minor deviations; and 5 indicates consistent and effective satisfaction of
the criterion throughout the trajectory. The judge assigns each dimension
separately rather than producing a single overall trajectory score.

We use the same judging prompt, rubric, and judge model across both
approaches and all harness settings. For each behavioral dimension, we first
average scores over the 200 trajectories within each held-out harness, and
then average equally across all 18 harnesses. For visualization, the
five-point scores are normalized by dividing by five, resulting in values
between 0.2 and 1.0 as reported in
Figures~\ref{fig:analysis_a} and~\ref{fig:analysis_b}.

\end{document}